\documentclass[10pt,twocolumn,letterpaper]{article}

\usepackage{cvpr}
\usepackage{times}
\usepackage{epsfig}
\usepackage{graphicx}
\usepackage{amsmath}
\usepackage{amssymb}
\usepackage{amsthm}
\usepackage[title]{appendix}

\usepackage{booktabs}
\usepackage{algorithm,algpseudocode}

\usepackage[breaklinks=true,bookmarks=false]{hyperref}

\newtheorem*{theorem*}{Theorem}
\newtheorem{proposition}{Proposition}
\newtheorem*{proposition*}{Proposition}
\DeclareMathOperator*{\argmax}{argmax}

\cvprfinalcopy 

\def\cvprPaperID{1162} 

\begin{document}

\title{Knowledge as Priors: Cross-Modal Knowledge Generalization\\ for Datasets without Superior Knowledge}

\author{Long Zhao\textsuperscript{1} \quad Xi Peng\textsuperscript{2} \quad Yuxiao Chen\textsuperscript{1} \quad Mubbasir Kapadia\textsuperscript{1} \quad Dimitris N. Metaxas\textsuperscript{1}\\
\textsuperscript{1}Rutgers University \quad \textsuperscript{2}University of Delaware\\
{\tt\small \{lz311,yc984,mk1353,dnm\}@cs.rutgers.edu, xipeng@udel.edu}
}

\maketitle

\begin{abstract}
   Cross-modal knowledge distillation deals with transferring knowledge from a model trained with superior modalities (Teacher) to another model trained with weak modalities (Student). Existing approaches require paired training examples exist in both modalities. However, accessing the data from superior modalities may not always be feasible. For example, in the case of 3D hand pose estimation, depth maps, point clouds, or stereo images usually capture better hand structures than RGB images, but most of them are expensive to be collected. In this paper, we propose a novel scheme to train the Student in a Target dataset where the Teacher is unavailable. Our key idea is to generalize the distilled cross-modal knowledge learned from a Source dataset, which contains paired examples from both modalities, to the Target dataset by modeling knowledge as priors on parameters of the Student. We name our method ``Cross-Modal Knowledge Generalization'' and demonstrate that our scheme results in competitive performance for 3D hand pose estimation on standard benchmark datasets.
\end{abstract}

\section{Introduction}\label{sec:introduction}

Leveraging multi-modal knowledge to boost the performance of classic computer vision problems, such as classification~\cite{ngiam2011multimodal,ramachandram2017deep,wang2019efficient}, object detection~\cite{gupta2016cross,song2016deep,xu2017learning} and gesture recognition~\cite{abavisani2019improving,chen2019construct,spurr2018cross,tian2018cr,yuan2019rgb,zhao2018learning,zhao2019semantic}, has emerged as a promising research field in recent years. Current paradigms for transferring knowledge across modalities involve aligning feature representations from multiple modalities of data during training, and then improving the performance of a unimodal system during testing with the aligned feature representations. Several different schemes for learning these feature representations have been proposed over the years~\cite{abavisani2019improving,spurr2018cross,wang2019distill,wang2019efficient}, and all of these rely on the availability of paired training samples from different modalities.

\begin{figure}[t]
\begin{center}
\includegraphics[width=\linewidth]{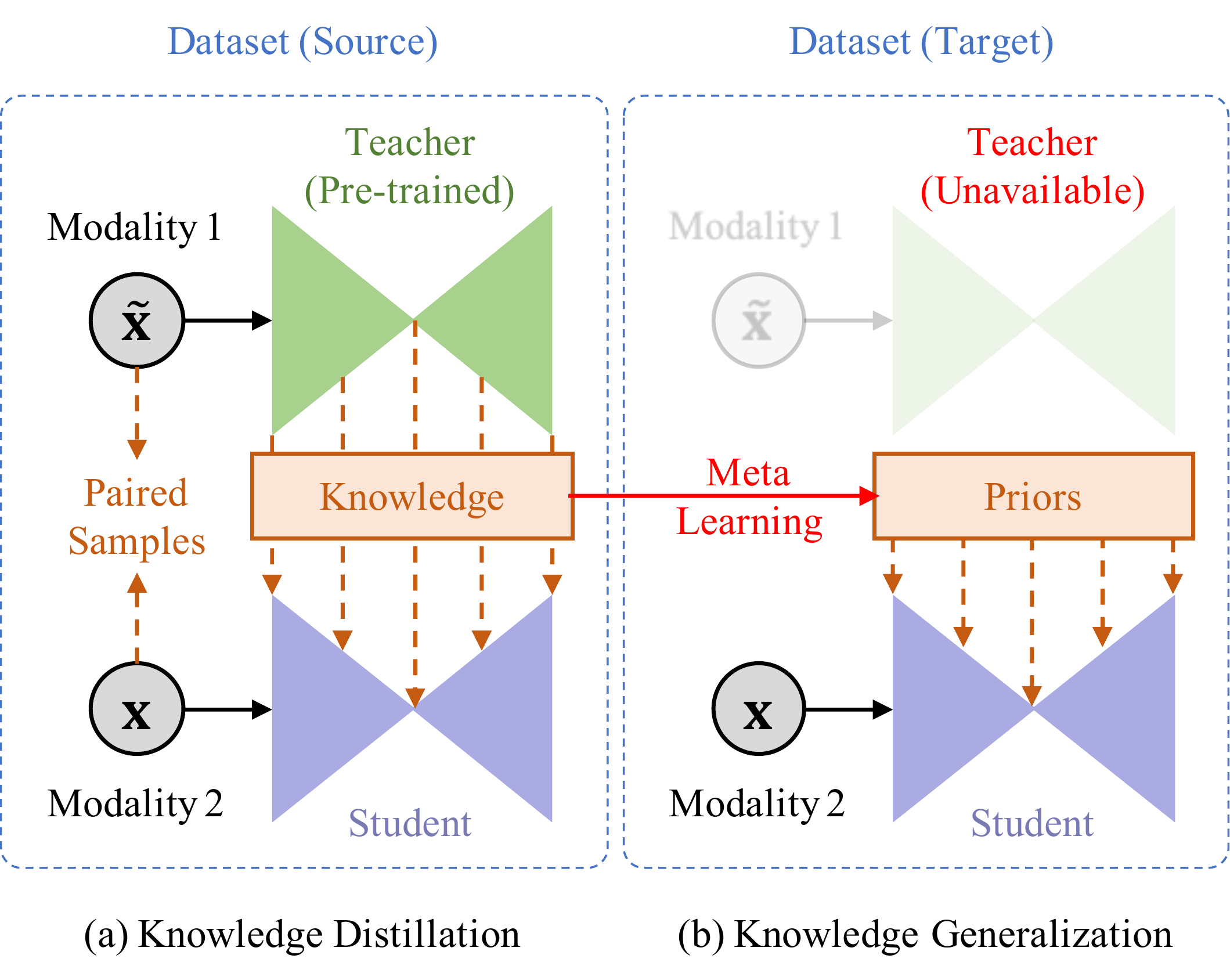}
\end{center}
\vspace{-.8em}
\caption{Cross-modal knowledge generalization. (a) Existing approaches distill cross-modal knowledge from the teacher to student in a source dataset. (b) We propose knowledge generalization which transfers learned knowledge in the source to a target dataset where the superior knowledge, \ie, the teacher, is unavailable.}
\label{fig:idea}
\end{figure}

Recently, Gupta~\etal~\cite{gupta2016cross} have introduced \textit{Cross-Modal Knowledge Distillation (CMKD)} which is a generic yet efficient scheme among these. They transfer knowledge across different modalities by a \textit{Teacher-Student} scheme~\cite{hinton2014distilling,srinivas2018knowledge,zagoruyko2017paying}. Generally, teacher networks deliver excellent performance since they are trained on modalities with superior knowledge. However, data of these modalities may be limited or expensive to be collected. On the other hand, a student network is trained using a weak modality and thereby often results in lower performance. The goal of knowledge distillation is to transfer superior knowledge from teachers to the student by aligning their intermediate feature representations. For simplicity, in this paper, we consider a form of cross-modal knowledge distillation problems in datasets where only two modalities, \ie, one teacher and one student, are involved as shown in Fig.~\ref{fig:idea}~(a).

The question we ask in this work is, \textit{what is the analogue of this paradigm for datasets which do not have modalities with superior knowledge?} As a motivating example, consider the case of 3D hand pose estimation. There are a number of ``superior'' modalities beyond RGB images which capture more accurate 3D hand structures, \eg, depth maps~\cite{qian2014realtime,tompson2014real,zimmermann2017learning}, point clouds~\cite{ge2018point,li2019point}, or stereo images~\cite{zhang2017hand}. These data together with their paired RGB images can be collected by corresponding devices or synthesized using pre-defined hand shape models~\cite{ge20193d,romero2017embodied}. However, most of real-world datasets still come with only a single weak modality, \ie, RGB images, which raises the question: \textit{is it possible for neural networks to transfer learned cross-modal knowledge to those target datasets where superior modalities are absent?}

We answer this question in this paper and propose a technique to transfer learned cross-modal knowledge from a source dataset, where both modalities are available, to the target dataset, where only one weak modality exists. Our technique uses ``paired'' data from the two modalities in the source dataset to distill cross-modal knowledge, and leverages meta-learning to generalize the knowledge to the target dataset by treating it as priors on the parameters of the student network. We call our scheme \textit{Cross-Modal Knowledge Generalization (CMKG)}, which is illustrated in Fig.~\ref{fig:idea}~(b). We further evaluate the performance of the proposed scheme in 3D hand pose estimation. We show that our generalized knowledge serves as a good regularizer to help the network learn better representations for 3D hands, and improves final results in the target dataset as well.

Our work makes the following contributions. First, unlike existing methods that distill knowledge across modalities in the same dataset, we introduce a novel method for Cross-Modal Knowledge Generalization, which generalizes the learned knowledge in the source to a target dataset where the superior modality is unavailable. Second, we introduce a novel meta-learning approach to transfer knowledge across datasets. Specifically, in Sect.~\ref{sec:distillation}, a simple yet powerful method is presented to distill cross-modal knowledge in the source dataset. The learned knowledge in the source dataset is then regarded as priors on network parameters during the training procedure in the target dataset. Sect.~\ref{sec:generalization} describes the meta-learning algorithm for learning these priors. Third, we comprehensively evaluate our scheme in 3D hand pose estimation and demonstrate its comparable performance to the state-of-the-art methods in Sect.~\ref{sec:experiments}. Note that our scheme can be easily generalized to different tasks, and we leave this for future work.

\section{Related Work}\label{sec:related_work}

\textbf{Knowledge Distillation.} The concept of knowledge distillation was first shown by Hinton~\etal~\cite{hinton2014distilling}. Subsequent research~\cite{ba2014deep,chen2017learning,romero2015fitnets} enhanced distillation by matching intermediate representations in the networks along with outputs using different approaches. Zagoruyko and Komodakis~\cite{zagoruyko2017paying} proposed to align attentional activation maps between networks. Srinivas and Fleuret~\cite{srinivas2018knowledge} improved it by applying Jacobian matching to networks. Recently, cross-modal knowledge distillation~\cite{gupta2016cross,wang2019efficient,yuan2019rgb} extended knowledge distillation by applying it to transferring knowledge across different modalities. Our approach generalizes cross-modal knowledge distillation to target datasets where superior modalities are missing.

\textbf{Meta-Learning.} Meta-learning is also known as ``learning to learn'', which intends to learn how learning can be performed in a more efficient manner. Previous approaches studied this problem from a
probabilistic modeling perspective~\cite{fe2003bayesian,lawrence2004learning} or in metric spaces~\cite{mishra2018simple,oreshkin2018tadam,snell2017prototypical}. Recent remarkable advances in gradient-based optimization approaches have rekindled the interest in meta-learning. Among these, Model-Agnostic Meta-Learning (MAML)~\cite{finn2017model} is proposed to solve few-shot learning. Li~\etal~\cite{li2018learning} extended MAML for domain generalization. Balaji~\etal~\cite{balaji2018metareg} introduced a meta-regularization function to train networks which can be easily generalized to different domains. Our meta-learning algorithm follows the spirit of these gradient-based methods but aims to learn cross-modal knowledge as priors.

\textbf{3D Hand Pose Estimation.} Estimating 3D hand poses from depth maps has made great progress in the past few years~\cite{ge2016robust,ge20173d,moon2018v2v,qian2014realtime,sun2015cascaded}. On the other hand, 3D hand pose estimation from RGB
images is significantly more challenging. Zimmermann and Brox~\cite{zimmermann2017learning} first proposed a deep network to learn a network-implicit 3D articulation prior together with 2D key points for predicting 3D hand poses. Other studies~\cite{spurr2018cross,yang2019aligning,yang2019disentangling} learned latent representations with a variational auto-encoder for inference of 3D hand poses. Note that some recent methods~\cite{boukhayma20193d,ge20193d,panteleris2018using,zhang2019end} focused on recovering the full shapes of 3D hands other than locations of key hand joints, which have a different research target compared with our work.

Yuan~\etal~\cite{yuan2019rgb} is the most related work in spirit to ours. Like our work, they employed cross-modal knowledge distillation to improve the performance of RGB-based 3D hand pose estimation. Our method differs significantly in that in addition to knowledge distillation, we aim to address a more challenging problem of transferring cross-modal knowledge to target datasets where depth maps are unavailable.

\section{Cross-Modal Knowledge Distillation}\label{sec:distillation}

We assume that the input data is available in two modalities $\mathbf{x}_i$ and $\tilde{\mathbf{x}}_i$, where $\tilde{\mathbf{x}}_i$ owns superior knowledge than $\mathbf{x}_i$. For each modality, one network is primarily trained with the data from its own modality. To be specific, we train a \textit{teacher} network $g$ using $\tilde{\mathbf{x}}_i$ and a \textit{student} network $f$ using $\mathbf{x}_i$. Given the ground truth $\mathbf{y}_i$, the teacher network parameterized by $\boldsymbol{\psi}$ minimizes the following $\ell^2$ regression loss:
\begin{equation}
\label{eq:dist_reg}
\mathcal{L}_\text{REG}(\tilde{\mathbf{x}}_i, \mathbf{y}_i; \boldsymbol{\psi}) = \left\| g(\tilde{\mathbf{x}}_i; \boldsymbol{\psi}) - \mathbf{y}_i \right\|^2.
\end{equation}

During the training of the student network, the goal of cross-modal knowledge distillation is to improve the learning process by transferring the knowledge from the teacher to student. The transferred knowledge can be viewed as an extra supervision in addition to the ground truth. To this end, the knowledge of networks is shared by aligning the semantics of the deep representations, \ie, activation maps of intermediate layers, between the teacher and student.

\begin{figure}[t]
\begin{center}
\includegraphics[width=1.\linewidth]{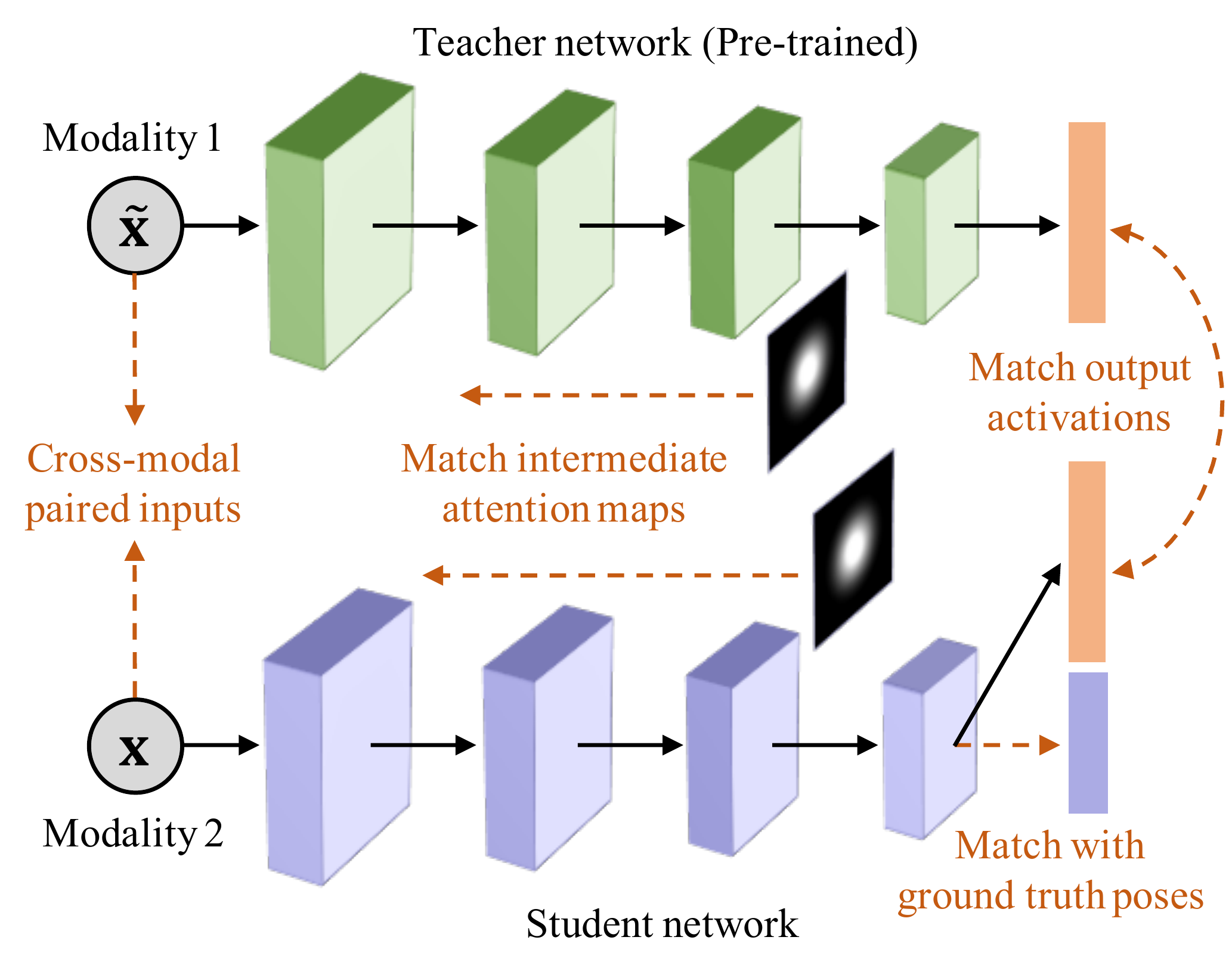}
\end{center}
\vspace{-.8em}
\caption{Illustration of our proposed approach for cross-modal knowledge distillation. For the student network, we match its outputs with the ground truth poses ($\mathcal{L}_\text{REG}$). Given cross-modal paired inputs, we match the final activations of a pre-trained teacher network ($\mathcal{L}_\text{ACT}$). We also match aggregated activations or ``attention'' maps between networks, similar to the work of~\cite{zagoruyko2017paying} ($\mathcal{L}_\text{ATT}$). The distillation loss ($\mathcal{L}_\text{DIST}$) is a combination of the last two.}
\label{fig:distillation}
\end{figure}

Let $Q_j \in \mathbb{R}^{C \times H \times W}$ denote the activation map of the $j$-th layer in the network, which consists of $C$ feature channels with spatial dimensions $H \times W$. We feed $\mathbf{x}_i$ to the student network $f$ and its paired $\tilde{\mathbf{x}}_i$ to the pre-trained teacher network $g$. Their last activation maps $Q_l$ are aligned by:
\begin{equation}
\label{eq:dist_act}
\mathcal{L}_\text{ACT}(\mathbf{x}_i, \tilde{\mathbf{x}}_i; \boldsymbol{\theta}) = \left\| Q_l(\mathbf{x}_i; f) - Q_l(\tilde{\mathbf{x}}_i; g) \right\|^2,
\end{equation}
where $\boldsymbol{\theta}$ are the parameters of the student network. Furthermore, we also match the attention maps~\cite{zagoruyko2017paying} of the intermediate layers between the teacher and student. Specifically, let $A_j \in \mathbb{R}^{H \times W}$ be the channel-wise attention map of $Q_j$ calculated by $A_j = \sum_{i=1}^{C}\|Q_j^{(i)}\|^2$, where $Q_j^{(i)}$ represents the $i$-th channel of $Q_j$. Then $A_j$ is $\ell^2$-normalized using $\bar{A}_j = \frac{A_j}{\|A_j\|}$, and we define the attention loss as:
\begin{equation}
\label{eq:dist_att}
\mathcal{L}_\text{ATT}(\mathbf{x}_i, \tilde{\mathbf{x}}_i; \boldsymbol{\theta}) = \sum_{i \in \mathcal{I}} \left\| \bar{A}_i(\mathbf{x}_i; f) - \bar{A}_i(\tilde{\mathbf{x}}_i; g) \right\|^2,
\end{equation}
where $\mathcal{I}$ denote the indices of all teacher-student activation layer pairs for which we want to transfer attention maps. Our full knowledge distillation loss can be written as:
\begin{equation}
\label{eq:dist_full}
\mathcal{L}_\text{DIST}(\mathbf{x}_i, \tilde{\mathbf{x}}_i; \boldsymbol{\theta}) = \mathcal{L}_\text{ACT} + \lambda \cdot \mathcal{L}_\text{ATT},
\end{equation}
where $\lambda$ is a hyper-parameter which is set to $1.0 \times 10^3$ empirically in the rest of the paper. The final student network is trained with the regression loss $\mathcal{L}_\text{REG}$ in Eq.~\eqref{eq:dist_reg} together with the distillation loss $\mathcal{L}_\text{DIST}$ in Eq.~\eqref{eq:dist_full}. The whole pipeline of our approach is summarized in Fig.~\ref{fig:distillation}.

\section{Cross-Modal Knowledge Generalization}\label{sec:generalization}

Consider two datasets: $\mathcal{D}_S = {\{\mathbf{x}^S_i, \tilde{\mathbf{x}}^S_i, \mathbf{y}^S_i\}}_i$ is a source dataset while $\mathcal{D}_T = {\{\mathbf{x}^T_i, \mathbf{y}^T_i\}}_i$ denotes a target dataset. Cross-modal knowledge can be efficiently distilled in the source dataset by neural networks as shown in Sect.~\ref{sec:distillation}, since training pairs $(\mathbf{x}^S_i, \tilde{\mathbf{x}}^S_i)$ are available in $\mathcal{D}_S$. However, due to the absence of $\mathcal{K}_T = \{\tilde{\mathbf{x}}^T_i\}_i$, direct knowledge distillation is impossible in the target dataset $\mathcal{D}_T$.

In this paper, we address a novel and challenging task of \textit{Cross-Modal Knowledge Generalization}. Specifically, we aim to learn the network parameters $\boldsymbol{\theta}_\text{DIST}$ which contain superior knowledge $\mathcal{K}_T$ in the target dataset $\mathcal{D}_T$. As mentioned above, the main challenge is that $\mathcal{K}_T$ is unavailable in $\mathcal{D}_T$. Our key idea is to generalize the learned knowledge from $\mathcal{D}_S$ to $\mathcal{D}_T$. This is achieved by interpreting knowledge as priors on the network parameters, which can be learned in $\mathcal{D}_S$ with meta-learning. In the following sections, we first derive our formulation from a probabilistic view. Then we present the meta-learning algorithm for knowledge generalization and theoretically show its connection to the expectation maximization (EM) algorithm.

\subsection{Knowledge as Priors}

From a Bayesian perspective, a neural network can be viewed as a probabilistic model $P(\mathbf{y}_i|\mathbf{x}_i, \boldsymbol{\theta})$: given an input $\mathbf{x}_i$, the network assigns a probability to each possible $\mathbf{y}_i \in \mathcal{Y}$ with the parameters $\boldsymbol{\theta}$. Here, we consider a regression problem where $P(\mathbf{y}_i|\mathbf{x}_i, \boldsymbol{\theta})$ is a Gaussian distribution which corresponds to a mean squared loss, and $\mathbf{x}_i$ is mapped onto the parameters of a distribution on $\mathcal{Y}$ using network layers parameterized by $\boldsymbol{\theta}$. Given a dataset $\mathcal{D} = {\{\mathbf{x}_i, \mathbf{y}_i\}}_i$, $\boldsymbol{\theta}$ can be learned by maximum likelihood estimation (MLE):
\begin{equation}
\begin{aligned}
\label{eq:mle}
\boldsymbol{\theta}_\text{MLE} &= \argmax_{\boldsymbol{\theta}} \log P(\mathcal{D}|\boldsymbol{\theta})\\
&= \argmax_{\boldsymbol{\theta}} \sum_{i} \log P(\mathbf{y}_i|\mathbf{x}_i, \boldsymbol{\theta}).
\end{aligned}
\end{equation}
We assume that $\log P(\mathcal{D}|\boldsymbol{\theta})$ is differentiable \wrt $\boldsymbol{\theta}$, and then Eq.~\eqref{eq:mle} is typically solved by gradient descent.

The objective of cross-modal knowledge generalization is to find the parameters $\boldsymbol{\theta}_\text{DIST}$ by using the training examples in $\mathcal{D}_T$ with intractable knowledge $\mathcal{K}_T$. This leads to maximize the posterior density of the parameters $\boldsymbol{\theta}$ directly depends on $\mathcal{D}_T$ and implicitly depends on $\mathcal{K}_T$. In order to explicitly capture this dependence, we introduce a latent variable $\boldsymbol{\phi}$ summarizing the knowledge carried by $\mathcal{K}_T$:
\begin{equation}
\begin{aligned}
\label{eq:gen}
P(\boldsymbol{\theta}|\mathcal{D}_T, {}&\mathcal{K}_T) = \int P(\boldsymbol{\theta}, \boldsymbol{\phi}|\mathcal{D}_T, \mathcal{K}_T) d\boldsymbol{\phi}\\
&= \int P(\boldsymbol{\theta}|\mathcal{D}_T, \mathcal{K}_T, \boldsymbol{\phi}) P(\boldsymbol{\phi}|\mathcal{D}_T, \mathcal{K}_T) d\boldsymbol{\phi}\\
&= \int P(\boldsymbol{\theta}|\mathcal{D}_T, \boldsymbol{\phi}) P(\boldsymbol{\phi}|\mathcal{D}_T, \mathcal{K}_T) d\boldsymbol{\phi}.
\end{aligned}
\end{equation}
Note that the last equation is the result of assuming that $\mathcal{K}_T$ and $\boldsymbol{\theta}$ are conditionally independent given the latent variable $\boldsymbol{\phi}$. Since both $\mathcal{K}_T$ and integrating Eq.~\eqref{eq:gen} over $\boldsymbol{\phi}$ are intractable, we make an approximation that uses a \textit{point estimation} $\boldsymbol{\phi}_{\text{META}}$. This point estimation is obtained via the meta-learning approach described in Sect.~\ref{sec:generalization:meta}, hence avoiding the need to perform integration over $\boldsymbol{\phi}$ or interact $\mathcal{K}_T$. Consequently, maximizing the logarithm of the posterior density of Eq.~\eqref{eq:gen} can be written as:
\begin{equation}
\begin{aligned}
\label{eq:gen_argmax}
\boldsymbol{\theta}_\text{DIST} &= \argmax_{\boldsymbol{\theta}} \log P(\boldsymbol{\theta}|\mathcal{D}_T, \mathcal{K}_T)\\
&\approx \argmax_{\boldsymbol{\theta}} \log P(\boldsymbol{\theta}|\mathcal{D}_T, \boldsymbol{\phi}_{\text{META}})\\
&= \argmax_{\boldsymbol{\theta}} \underbrace{\log P(\mathcal{D}_T|\boldsymbol{\theta})}_{\text{Likelihood}} + \underbrace{\log P(\boldsymbol{\theta}|\boldsymbol{\phi}_{\text{META}})}_{\text{Prior (Knowledge)}},
\end{aligned}
\end{equation}
where the last equality results from a direct application of Bayes rule. So, finding the parameters $\boldsymbol{\theta}_\text{DIST}$ involves a two step training procedure: (1) optimizing the prior term which obtains the point estimation $\boldsymbol{\phi}_{\text{META}}$ using meta-learning and (2) optimizing the likelihood term which maximizes Eq.~\eqref{eq:gen_argmax} using the learned parameters $\boldsymbol{\phi}_{\text{META}}$.

In a Bayesian setting, priors on the parameters can be interpreted as regularization. Thus the prior term in Eq.~\eqref{eq:gen_argmax} is implemented as a regularizer during network training. Several other regularization schemes have been proposed in the literature such as weight decay~\cite{krogh1992simple}, dropout~\cite{srivastava2014dropout,wan2013regularization} and batch normalization~\cite{ioffe2015batch}. While they aim to reduce error on examples drawn from the test distribution, the objective of our work is to learn a regularizer that captures cross-modal knowledge learned from the source dataset.

\subsection{Learning Priors with Meta-Learning}
\label{sec:generalization:meta}

As mentioned above, we model the prior term as a regularizer $\mathcal{R}(\boldsymbol{\theta};\boldsymbol{\phi})$. Given the input $\boldsymbol{\theta}$, $\mathcal{R}$ is implemented with a neural network parameterized by $\boldsymbol{\phi}$.

\begin{figure}[t]
\begin{center}
\includegraphics[width=1.\linewidth]{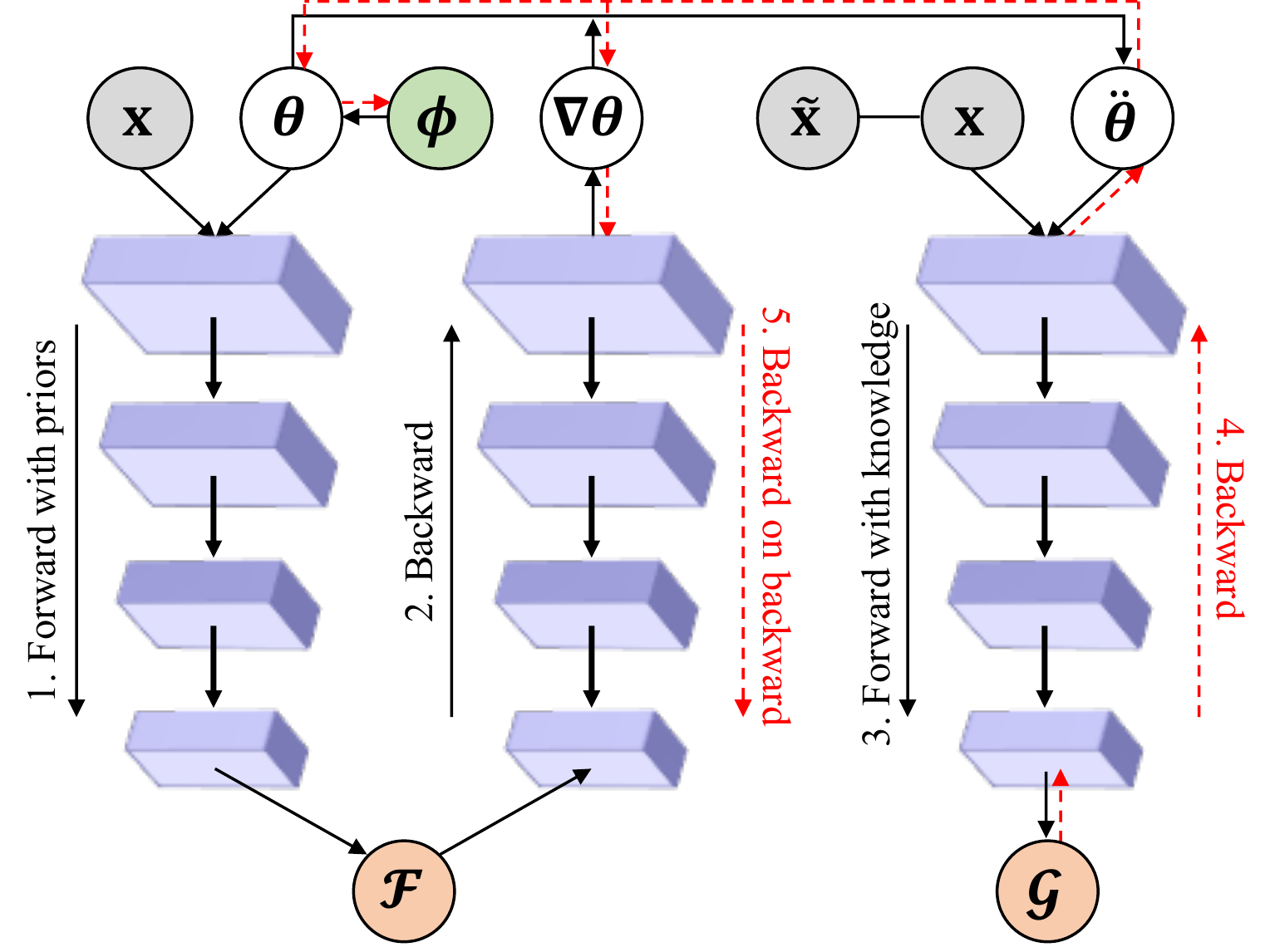}
\end{center}
\vspace{-.8em}
\caption{Computational graph of our meta-training algorithm (as shown in Algorithm~\ref{alg:meta_training}) in a deep neural network, which can be efficiently implemented using the second order derivative.}
\label{fig:comp_graph}
\end{figure}

As described in Sect.~\ref{sec:distillation}, cross-modal knowledge distillation leads to optimize the following objective:
\begin{equation}
\label{eq:meta_dist}
\mathcal{G}(\mathbf{x}_i, \tilde{\mathbf{x}}_i, \mathbf{y}_i; \boldsymbol{\theta}) = \mathcal{L}_\text{REG}(\mathbf{x}_i, \mathbf{y}_i; \boldsymbol{\theta}) + \mathcal{L}_\text{DIST}(\mathbf{x}_i, \tilde{\mathbf{x}}_i; \boldsymbol{\theta}),
\end{equation}
where $\mathcal{L}_\text{REG}$ is the regression loss minimizing the mean squared errors of the prediction and ground truth, and the distillation loss $\mathcal{L}_\text{DIST}$ distills knowledge from the teacher to student. Using the regularizer $\mathcal{R}$, we introduce a regularized regression loss which is defined as:
\begin{equation}
\label{eq:meta_reg}
\mathcal{F}(\mathbf{x}_i, \mathbf{y}_i; \boldsymbol{\theta}, \boldsymbol{\phi}) = \mathcal{L}_\text{REG}(\mathbf{x}_i, \mathbf{y}_i; \boldsymbol{\theta}) + \mathcal{R}(\boldsymbol{\theta};\boldsymbol{\phi}).
\end{equation}
During the training procedure on the source dataset, we aim to learn the regularizer $\mathcal{R}$ in Eq.~\eqref{eq:meta_reg} which mimics the behavior of $\mathcal{L}_\text{DIST}$ in Eq.~\eqref{eq:meta_dist}, so that $\mathcal{F}$ can be applied to the target dataset where the superior knowledge is missing. We now describe the procedure for learning $\mathcal{R}$.

When training the student network on the source dataset, at iteration $k$, we begin by sampling a mini-batch from the dataset. Using this batch, $l$ steps of gradient descent are first performed with the regularized regression loss $\mathcal{F}$. Let $\ddot{\boldsymbol{\theta}}_k$ denote the network parameters after these $l$ steps. Then the full loss $\mathcal{G}$ on the same batch computed using $\ddot{\boldsymbol{\theta}}_k$ is minimized \wrt the regularizer parameters $\boldsymbol{\phi}$. The regularizer $\mathcal{R}$ is finally updated with the gradients which unroll through the $l$ gradient steps. This ensures that $\mathcal{G}$ can be approximated by $\mathcal{F}$ using $\mathcal{R}$. After finishing the training, since the same regularizer is trained on every pair of $(\mathbf{x}_i, \tilde{\mathbf{x}}_i)$, the resulting $\mathcal{R}$ captures the notion of cross-modal knowledge contained in the source dataset. Please refer to Fig.~\ref{fig:comp_graph} for an illustration of the meta-training step. The entire algorithm is given in Algorithm~\ref{alg:meta_training}. Note that $l$ is set to 1 empirically in this paper, as we observe that a $1$-step update is sufficient to achieve good performance.

\begin{algorithm}[t]
\begin{algorithmic}[1]
\Require Batch size $N$, \# of iterations $K$, learning rate $\alpha$.
\Require \# of inner iterations $l$, meta learning rate $\beta$.
\State Initialize $\boldsymbol{\theta}_0$, $\boldsymbol{\phi}_0$
\For{$k=0$ to $K-1$}
    \State Sample $N$ examples $\{(\mathbf{x}^S_n, \tilde{\mathbf{x}}^S_n, \mathbf{y}^S_n) \sim \mathcal{D}_S\}_{n = 1}^N$
    \State $\ddot{\boldsymbol{\theta}}_0 \gets \boldsymbol{\theta}_k$
    \For{$i=0$ to $l-1$}
        \State $\ddot{\boldsymbol{\theta}}_{i+1} \gets \ddot{\boldsymbol{\theta}}_i - \alpha\nabla_{\ddot{\boldsymbol{\theta}}_i}\mathcal{F}(\mathbf{x}^S_n, \mathbf{y}^S_n; \ddot{\boldsymbol{\theta}}_i, \boldsymbol{\phi}_k)$ \Comment{E-step}
    \EndFor
    \State $\ddot{\boldsymbol{\theta}}_k \gets \ddot{\boldsymbol{\theta}}_l$
    \State $\boldsymbol{\phi}_{k+1} \gets \boldsymbol{\phi}_k - \beta\nabla_{\boldsymbol{\phi}_k}\mathcal{G}(\mathbf{x}^S_n, \tilde{\mathbf{x}}^S_n, \mathbf{y}^S_n; \ddot{\boldsymbol{\theta}}_k)$ \Comment{M-step}
    \State $\boldsymbol{\theta}_{k+1} \gets \boldsymbol{\theta}_k - \alpha\nabla_{\boldsymbol{\theta}_k}\mathcal{G}(\mathbf{x}^S_n, \tilde{\mathbf{x}}^S_n, \mathbf{y}^S_n; \boldsymbol{\theta}_k)$
\EndFor
\State $\boldsymbol{\phi}_{\text{META}} \gets \boldsymbol{\phi}_K$
\end{algorithmic}
\caption{Meta-training for learning priors.}
\label{alg:meta_training}
\end{algorithm}

Once the regularizer is learned, its parameters $\boldsymbol{\phi}_{\text{META}}$ are frozen and the final student network
initialized from scratch is trained on the target dataset using the regularized loss function $\mathcal{F}$. This meta-testing procedure generalizes the learned knowledge to the target dataset with $\mathcal{R}$ parameterized by $\boldsymbol{\phi}_{\text{META}}$ as summarized in Algorithm~\ref{alg:meta_testing}.

Our meta-learning approach is general and can be implemented by any type of regularizer. In this paper, we use weighted $\ell^2$ loss as our regularization function:
\begin{equation}
\label{eq:reg}
\mathcal{R}(\boldsymbol{\theta}; \boldsymbol{\phi}) = \sum_i \phi_i \|\theta_i\|^2,
\end{equation}
where $\phi_i$ and $\theta_i$ are the $i$-th weight and parameter of the network. The use of weighted $\ell^2$ loss can be interpreted as a learnable weight decay mechanism: weights $\theta_i$ for which $\phi_i$ is large will be decayed to zero and those for which $\phi_i$ is small will be boosted. By using our meta-learning approach, we select a set of weights that carry cross-modal knowledge across every pair of inputs $(\mathbf{x}_i, \tilde{\mathbf{x}}_i)$.

\subsection{Theoretical Understanding}

This section gives a theoretical understanding of Algorithm~\ref{alg:meta_training} in Sect.~\ref{sec:generalization:meta}. We draw its connection to the expectation maximization (EM) algorithm and thus its convergence is theoretically guaranteed. To achieve this, we first derive the lower bound of the target objective and then show how it is solved by our meta-learning algorithm using EM.

In a Bayesian framework, given the evidence $\mathcal{D}_S$, learning the parameters $\boldsymbol{\phi}$ of priors leads to maximize the likelihood $P(\mathcal{D}_S|\boldsymbol{\phi})$. Proposition~\ref{thm:elbo} indicates its lower bound.

\begin{proposition}
\label{thm:elbo}
Let $q$ be any posterior distribution function over the latent variables $\boldsymbol{\theta}$ given the evidence $\mathcal{D}_S$. Then, the marginal log-likelihood can be lower bounded:
\begin{equation}
\label{eq:elbo_bound}
\log P(\mathcal{D}_S|\boldsymbol{\phi}) = \log \int P(\mathcal{D}_S, \boldsymbol{\theta}|\boldsymbol{\phi}) d\boldsymbol{\theta} \geq \mathcal{E}(q,\boldsymbol{\phi}),
\end{equation}
where $\mathcal{E}$ is the evidence lower-bound (ELBO) defined as:
\begin{equation}
\label{eq:elbo_def}
\mathcal{E} \triangleq \mathbb{E}_q [ \log P(\mathcal{D}_S|\boldsymbol{\theta})]\\ - \textup{KL}[ q(\boldsymbol{\theta}|\mathcal{D}_S) \| P(\boldsymbol{\theta}|\boldsymbol{\phi})].
\end{equation}
\end{proposition}

Note that $\text{KL}[\cdot \| \cdot]$ in Eq.~\eqref{eq:elbo_def} represents the KL divergence between two distributions $q$ and $P$. The proof to this proposition can be found in our supplementary material. According to Proposition~\ref{thm:elbo}, the following proposition shows that Algorithm~\ref{alg:meta_training} is an instance of EM maximizing $\mathcal{E}$.

\begin{proposition}
\label{thm:em}
The parameters $\boldsymbol{\phi}$ can be estimated by maximizing the evidence lower-bound of $\log P(\mathcal{D}_S|\boldsymbol{\phi})$ via expectation maximization (EM) as shown in Algorithm~\ref{alg:meta_training}.
\end{proposition}
\begin{proof}
The EM algorithm can be viewed as two alternating maximization steps: E-step and M-step. In the $k$-th E-step, for fixed $\boldsymbol{\phi}$, the objective $\mathcal{E}$ is bounded above by the first term in Eq.~\eqref{eq:elbo_def}, and achieves that bound when the KL divergence term is zero. This is achieved if and only if $q$ is equal to $P$. Therefore, the E-step sets $q$ to $P$ and estimates the posterior probability:
\begin{equation}
\ddot{\boldsymbol{\theta}}_k = \argmax_{\boldsymbol{\theta}_k} q_k = \argmax_{\boldsymbol{\theta}_k} P(\boldsymbol{\theta}_k|\boldsymbol{\phi}_k).
\end{equation}
And, after an E-step, the objective $\mathcal{E}$ equals the likelihood term. In the $k$-th M-step, we fix $\ddot{\boldsymbol{\theta}}$ and solve:
\begin{equation}
\boldsymbol{\phi}_{k+1} = \argmax_{\boldsymbol{\phi}_k} \mathcal{E}(q_k, \boldsymbol{\phi}_k).
\end{equation}
Both E-step and M-step are solved by gradient descent as commented in Algorithm~\ref{alg:meta_training}. We have thus shown that Algorithm~\ref{alg:meta_training} is an instance of EM.
\end{proof}

\begin{algorithm}[t]
\begin{algorithmic}[1]
\Require Batch size $N$, \# of iterations $K$, learning rate $\alpha$.
\Require Learned parameters $\boldsymbol{\phi}_{\text{META}}$ from Algorithm~\ref{alg:meta_training}.
\State Initialize $\boldsymbol{\theta}_0$
\For{$k=0$ to $K-1$}
    \State Sample $N$ examples $\{(\mathbf{x}^T_n, \mathbf{y}^T_n) \sim \mathcal{D}_T\}_{n = 1}^N$
    \State $\boldsymbol{\theta}_{k+1} \gets \boldsymbol{\theta}_k - \alpha\nabla_{\boldsymbol{\theta}_k}\mathcal{F}(\mathbf{x}^T_n, \mathbf{y}^T_n; \boldsymbol{\theta}_k, \boldsymbol{\phi}_{\text{META}})$
\EndFor
\State $\boldsymbol{\theta}_{\text{DIST}} \gets \boldsymbol{\theta}_K$
\end{algorithmic}
\caption{Meta-testing for knowledge generalization.}
\label{alg:meta_testing}
\end{algorithm}

\section{Experiments}\label{sec:experiments}

The proposed approach is evaluated in 3D hand pose estimation. We aim to answer the following three questions: (1) Can our Cross-Modal Knowledge Distillation (CMKD) distill accurate cross-modal knowledge from the source dataset?~(Sect.~\ref{sec:exp:dist}) (2) Does the proposed Cross-Modal Knowledge Generalization (CMKG) successfully transfer learned knowledge to the target dataset?~(Sect.~\ref{sec:exp:gen}) (3) And what factors influence the effect of our CMKG?~(Sect.~\ref{sec:exp:dis})

\subsection{Implementation Details}\label{sec:exp:imp}

For simplicity, we use the same architecture for teacher and student networks. We choose ResNet~\cite{he2016deep} as the backbone, and adjust the final fully connected layer to output a vector representing the 3D positions of 21 hand joints. All corresponding depth maps of RGB images in the dataset are employed as the modality containing superior knowledge.

\textbf{Data Augmentation.} Recent methods~\cite{boukhayma20193d,ge20193d,varol2017learning} show that learning from synthetic data improves the performance of 3D pose estimation, as it offers more effective hand variations than traditional data augmentation technologies, \eg, random cropping and rotation. Hence, we create a synthetic dataset of paired hand images and depth maps with their 3D annotations using the MANO~\cite{romero2017embodied} hand model for synthetic data augmentation. Following the setting of~\cite{boukhayma20193d}, hand geometries are obtained by sampling pose and shape parameters from $[-2, 2]^{10}$ and $[-0.03, 0.03]^{10}$, respectively. Meanwhile, hand appearances are modeled by the original scans with 3D coordinates and RGB values from~\cite{romero2017embodied}. We create example hand appearances using these registered scan topologies. After rotations, translations and scalings are applied to hand models, the textured hands are finally rendered on background images which are randomly sampled and cropped from~\cite{jegou2008hamming,lin2014microsoft}. In total, we synthesize 50,000 hand images with large variations for training.

\textbf{Network Training.} The input image is resized to $256 \times 256$. For CMKD, all networks are trained using Adam~\cite{kingma2014adam} with mini-batches of size 32. The learning rate is set as $2.5 \times 10^{-4}$. The teacher is pre-trained for 200 epochs, while the student is trained with only the regression loss for 100 epochs and then fine-tuned with the full loss for another 100 epochs. For CMKG, the regularizer is optimized using Stochastic Gradient Descent (SGD) with the learning rate of $1.0 \times 10^{-3}$ during the fine-tuning of the student network.

\subsection{Datasets and Metrics}

Our proposed approach is comprehensively evaluated on two publicly available datasets for 3D hand pose estimation: RHD~\cite{zimmermann2017learning} and STB~\cite{zhang2017hand} with the standard metrics.

\textbf{Datasets.} Rendered Hand Pose Dataset (RHD)~\cite{zimmermann2017learning} is a synthetic dataset built upon 20 different characters performing 39 actions. It provides 41,258 images for training and 2,728 images for evaluation with a resolution of $320 \times 320$. All of them are fully annotated with a 21 joint skeleton hand model and additionally the depth map for each hand. This dataset is challenging due to the large variations in viewpoints and textures. We employ RHD for training and evaluating our knowledge distillation method.

Stereo Hand Pose Tracking Benchmark (STB)~\cite{zhang2017hand} is a real-world dataset which contains 18,000 stereo image pairs as well as the ground truth 3D positions of 21 hand joints from different scenarios. This benchmark has 12 different sequences and every sequence contains 1,500 stereo pairs. Following the evaluation protocol of~\cite{cai2018weakly,ge20193d,spurr2018cross,zimmermann2017learning}, we use the sequence of B1 for evaluation and the others for training. STB is utilized for evaluating the proposed cross-modal knowledge generation algorithm.

To make the joint definition consistent across different datasets, we reorganize the joints of each hand according to the layout of MANO~\cite{romero2017embodied}. Especially, we move the root joint location from palm center to wrist of each hand in STB. Following the
same protocol used in~\cite{cai2018weakly,ge20193d,spurr2018cross,zimmermann2017learning}, the absolute depth of root joint (wrist) and global hand scale, which is set as the bone length between MCP and PIP joints of the middle finger, are provided at test time.

\textbf{Metrics.} We evaluate the performance of 3D hand pose estimation with three common metrics in the literature: (1) EPE: the mean hand joint error which measures the average Euclidean distance in millimeters (mm) between the predicted 3D joints and the ground truth; (2) 3D PCK: the percentage of correct key points which are within the Euclidean distance of a certain threshold to its respective ground truth position; (3) AUC: the area under the curve on 3D PCK for different error thresholds.

\subsection{Evaluation of Knowledge Distillation}\label{sec:exp:dist}

\begin{table}[t]
\begin{center}
\resizebox{.95\linewidth}{!}{
\begin{tabular}{l|l|c}
\toprule
Settings & Backbone & EPE (RGB / Depth / KD)\\
\midrule
$\mathcal{L}_\text{ACT}$ & ResNet-18 & $24.68$ / $13.60$ / $23.41_{\downarrow 1.27}$ \\
$\mathcal{L}_\text{ACT}$, $\mathcal{L}_\text{ATT}$ & ResNet-18 & $24.68$ / $13.60$ / $22.19_{\downarrow 2.49}$ \\
$\mathcal{L}_\text{ACT}$, $\mathcal{L}_\text{ATT}$, $\mathcal{A}$ & ResNet-18 & $23.07$ / $12.06$ / $20.89_{\downarrow 2.18}$ \\
$\mathcal{L}_\text{ACT}$, $\mathcal{L}_\text{ATT}$, $\mathcal{A}$ & ResNet-50 & $\underline{20.74}$ / $\underline{10.78}$ / $\underline{18.06}_{\downarrow 2.68}$ \\
\bottomrule
\end{tabular}}
\end{center}
\caption{Ablation study on the choices of loss terms used in Eq.~\eqref{eq:dist_full}, synthetic data augmentation denoted by $\mathcal{A}$, and network backbone for knowledge distillation. We also report the performance gain in EPE (mm) obtained by cross-modal knowledge distillation.}
\label{tbl:dist_ablation}
\end{table}

To evaluate the performance of the proposed knowledge distillation approach for 3D hand pose estimation, we train three networks for each setting: a baseline network trained with RGB images (RGB), a teacher network trained with depth maps (Depth) and a student network trained using the knowledge distillation algorithm presented in Sect.~\ref{sec:distillation} (KD). All the experiments are conducted on RHD dataset.

\textbf{Ablation Study.} We first evaluate the impacts of different losses used in knowledge distillation, data augmentation, and network architecture on the performance of 3D hand pose estimation. The results of EPE are presented in Table~\ref{tbl:dist_ablation}. We can see that the model trained with the full distillation loss ($\mathcal{L}_\text{ACT}$ and $\mathcal{L}_\text{ATT}$) achieves higher performance improvement, from 1.27 (mm) to 2.49 (mm), which indicates that all the losses have contributions to distilling cross-modal knowledge from depth maps for 3D hand pose estimation. Moreover, synthetic data augmentation and employing deeper network during the training procedure can further boost the performance.

\begin{figure*}[t]
\begin{center}
\includegraphics[height=15.5em]{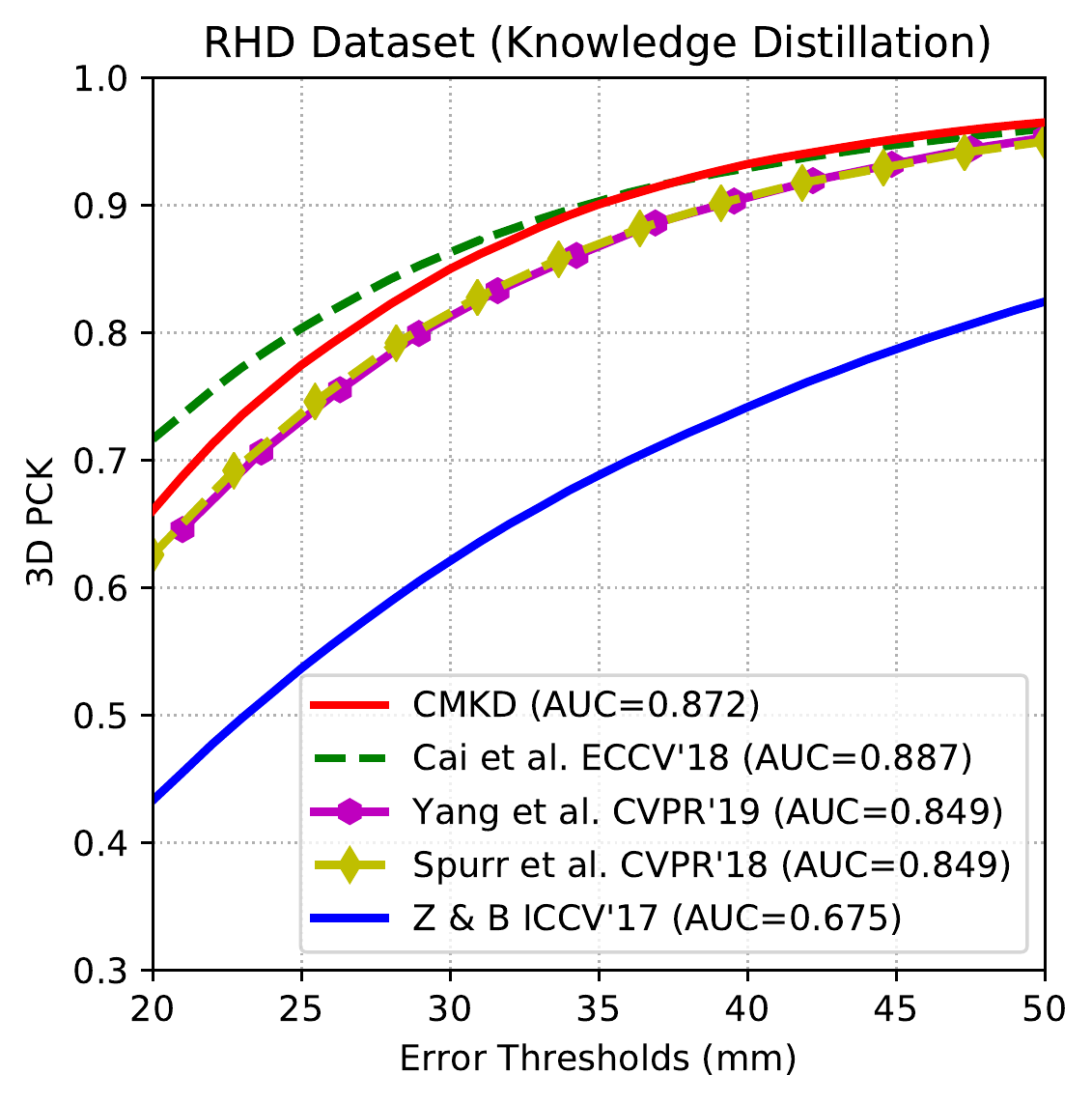}
\includegraphics[height=15.5em]{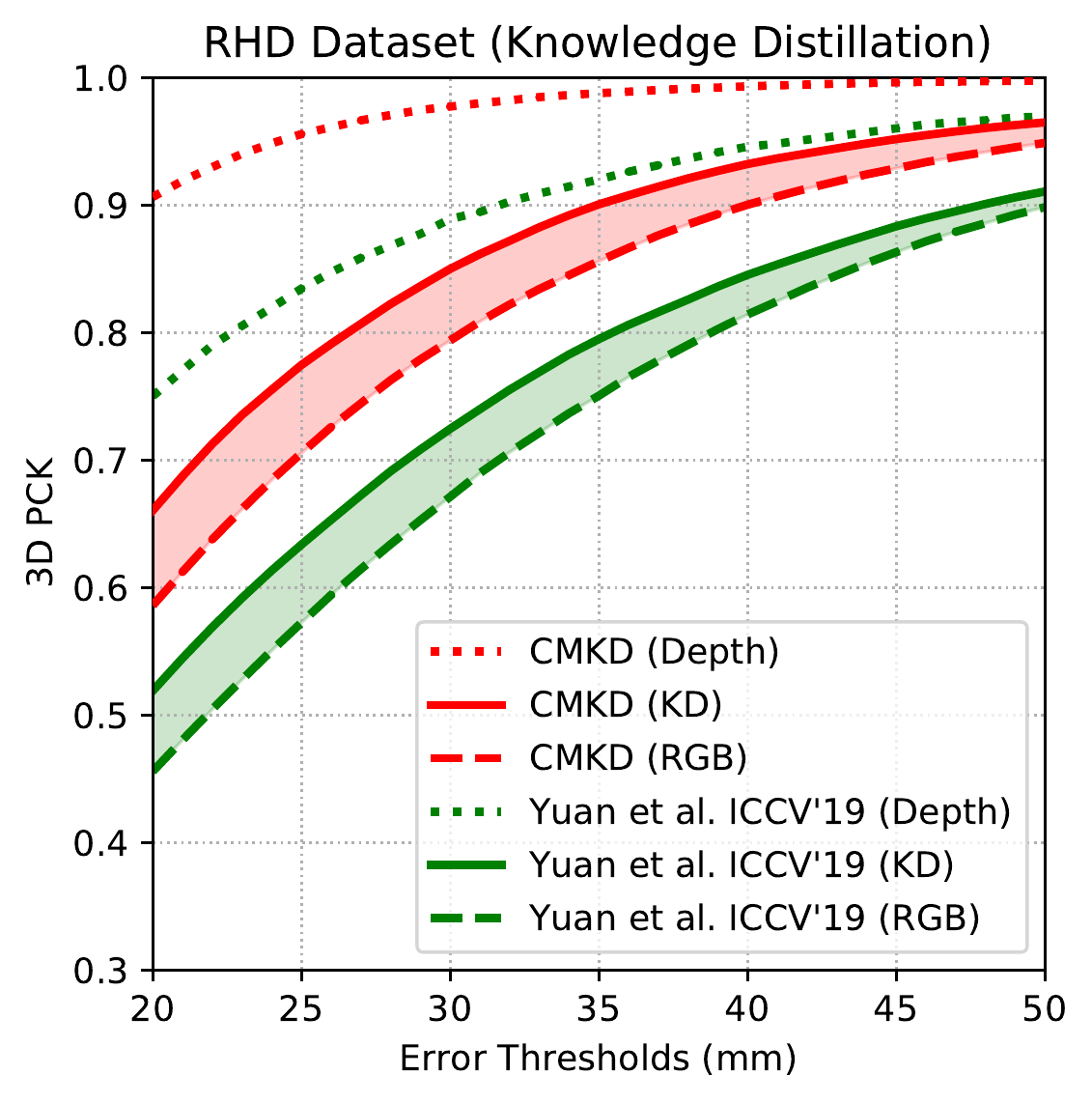}
\includegraphics[height=15.5em]{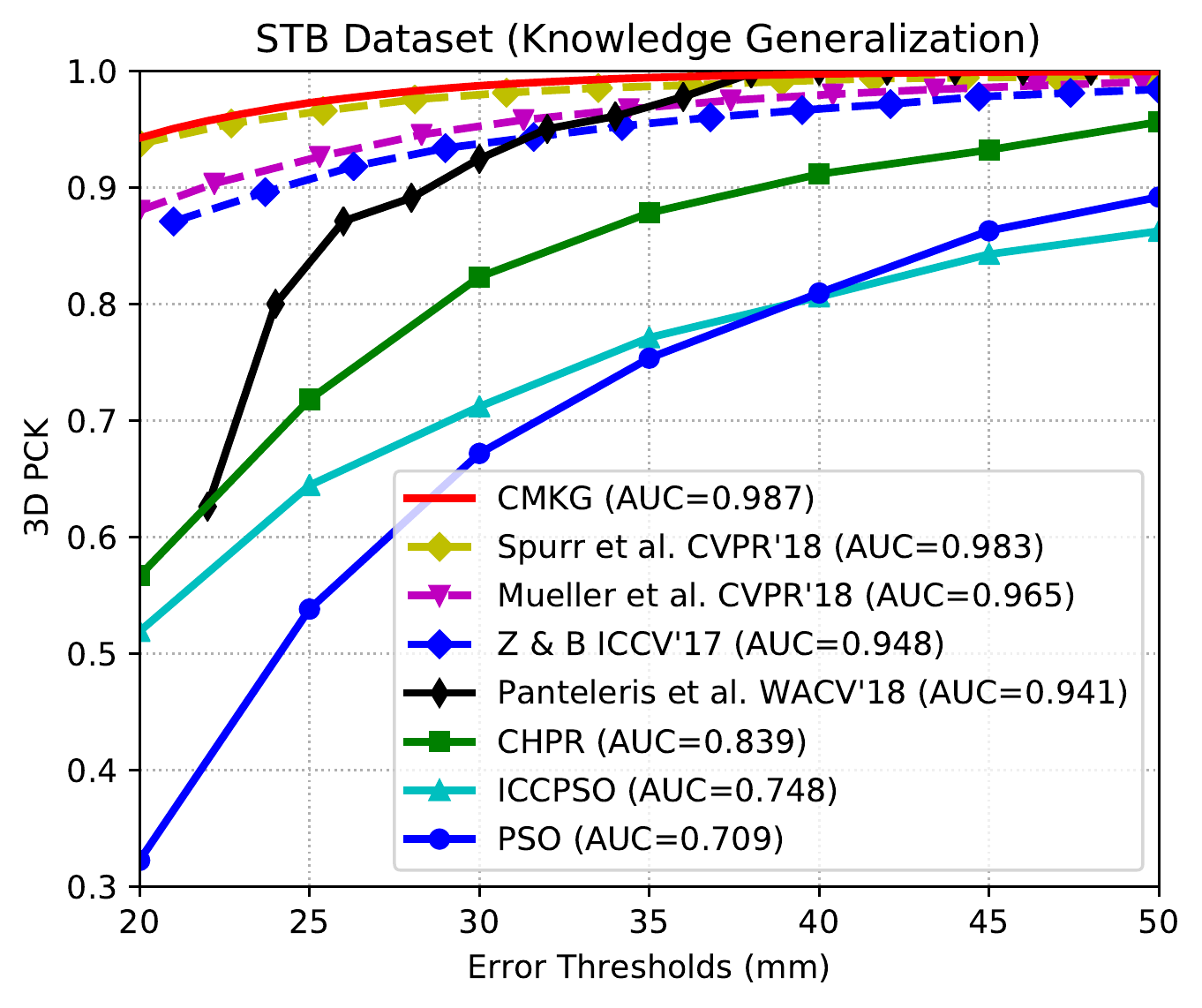}
\end{center}
\vspace{-.8em}
\caption{Comparisons with state of the art. Left: 3D PCK on RHD~\cite{zimmermann2017learning} of our knowledge distillation approach (CMKD). Our method has comparable performance to Cai~\etal~\cite{cai2018weakly} which relies on additional 2D annotations for network training. Middle: Comparison with Yuan~\etal~\cite{yuan2019rgb} which also distills knowledge from depth. Our approach obtains a more significant improvement (red area, $\Delta_\text{AUC} = 0.045$) than~\cite{yuan2019rgb} (green area, $\Delta_\text{AUC} = 0.041$). Right: Our knowledge generalization method (CMKG) obtains state-of-the-art results on STB~\cite{zhang2017hand}.}
\label{fig:dist_comparison}
\end{figure*}

\textbf{Comparison to State of the Art.} We compare the 3D PCK curves with state-of-the-art methods~\cite{cai2018weakly,spurr2018cross,yang2019disentangling,yuan2019rgb,zimmermann2017learning} on RHD dataset in Fig.~\ref{fig:dist_comparison}. We use ResNet-50 as the backbone. Note that some other works~\cite{boukhayma20193d,ge20193d,zhang2019end} aim to predict the 3D hand shape other than hand joint locations, which are with different research targets compared with ours. Therefore, they are not included here. In Fig.~\ref{fig:dist_comparison} (left), our method surpasses most existing methods except~\cite{cai2018weakly}, which has a higher AUC of 0.015. However, it is not directly comparable, as~\cite{cai2018weakly} incorporates 2D annotations as an additional supervision during network training.

In Fig.~\ref{fig:dist_comparison} (middle), we further compare our approach to Yuan~\etal~\cite{yuan2019rgb} which is the most related work also distilling cross-modal knowledge from depth maps for 3D hand pose estimation. We can find that our method substantially outperforms~\cite{yuan2019rgb}. More importantly, the performance gain achieved by our approach ($\Delta_\text{AUC} = 0.045$) is larger than~\cite{yuan2019rgb} ($\Delta_\text{AUC} = 0.041$), which shows that the proposed knowledge distillation algorithm is more efficient.

\subsection{Evaluation of Knowledge Generalization}\label{sec:exp:gen}

In order to evaluate the effectiveness of the proposed knowledge generalization algorithm, we transfer the learned cross-modal knowledge in RHD to STB and compare our approach to other regularization functions.

\textbf{Effect of Regularizers.} In this experiment, we study the effect of different regularizers including the proposed $\mathcal{R}$ in Eq.~\eqref{eq:reg} on the performance of network trained on STB. We compare our formulation with the default regularizers commonly used in the literature: $\sigma \sum_{i} \| \phi_i \|^p$, where $\| \cdot \|^p$ is the $p$-norm of the parameter and $\sigma$ is a constant weight manually selected for each network. We experiment on the $\ell^1$ and $\ell^2$ regularizers (where $p$ equals 1 or 2, respectively) and different choices of $\sigma$. We also implement a variant of the proposed $\mathcal{R}$ which is $\ell^1$-regularized. The performance of these regularizers are reported in Table~\ref{tbl:gen_ablation}. We observe that our proposed regularizers outperform the default regularization functions by a large margin. Especially, our $\ell^2$-regularized $\mathcal{R}$ achieves the best performance. These results demonstrate that $\mathcal{R}$ carries effective knowledge learned from the source dataset which helps the training of the target network.

\begin{table}[t]
\begin{center}
\resizebox{.75\linewidth}{!}{
\begin{tabular}{l|c|c}
\toprule
Regularizer & EPE (mm) & AUC\\
\midrule
None & $15.67$ & $0.915$ \\
$\ell^1$, $\sigma = 1.0 \times 10^{-4}$ & $11.41_{\downarrow 4.26}$ & $0.972_{\uparrow 0.057}$ \\
$\ell^1$, $\sigma = 1.0 \times 10^{-6}$ & $11.82_{\downarrow 3.85}$ & $0.964_{\uparrow 0.049}$ \\
$\ell^2$, $\sigma = 1.0 \times 10^{-3}$ & $12.28_{\downarrow 3.39}$ & $0.957_{\uparrow 0.042}$ \\
$\ell^2$, $\sigma = 1.0 \times 10^{-5}$ & $12.02_{\downarrow 3.65}$ & $0.964_{\uparrow 0.049}$ \\
\midrule
$\mathcal{R}$, $\ell^1$-regularized & $\;\: 8.86_{\downarrow 6.81}$ & $0.985_{\uparrow 0.070}$ \\
$\mathcal{R}$, $\ell^2$-regularized & $\;\: \underline{8.18}_{\downarrow 7.49}$ & $\underline{0.987}_{\uparrow 0.072}$ \\
\bottomrule
\end{tabular}}
\end{center}
\caption{Effect of different classes of regularization functions on STB~\cite{zhang2017hand}. Note that $\sigma$ denotes the constant weight manually chosen for the default $\ell^1$ or $\ell^2$ regularizer. We report EPE (mm) and AUC together with the performance gain for each method.}
\label{tbl:gen_ablation}
\end{table}

\begin{figure}[t]
\begin{center}
\includegraphics[width=.495\linewidth]{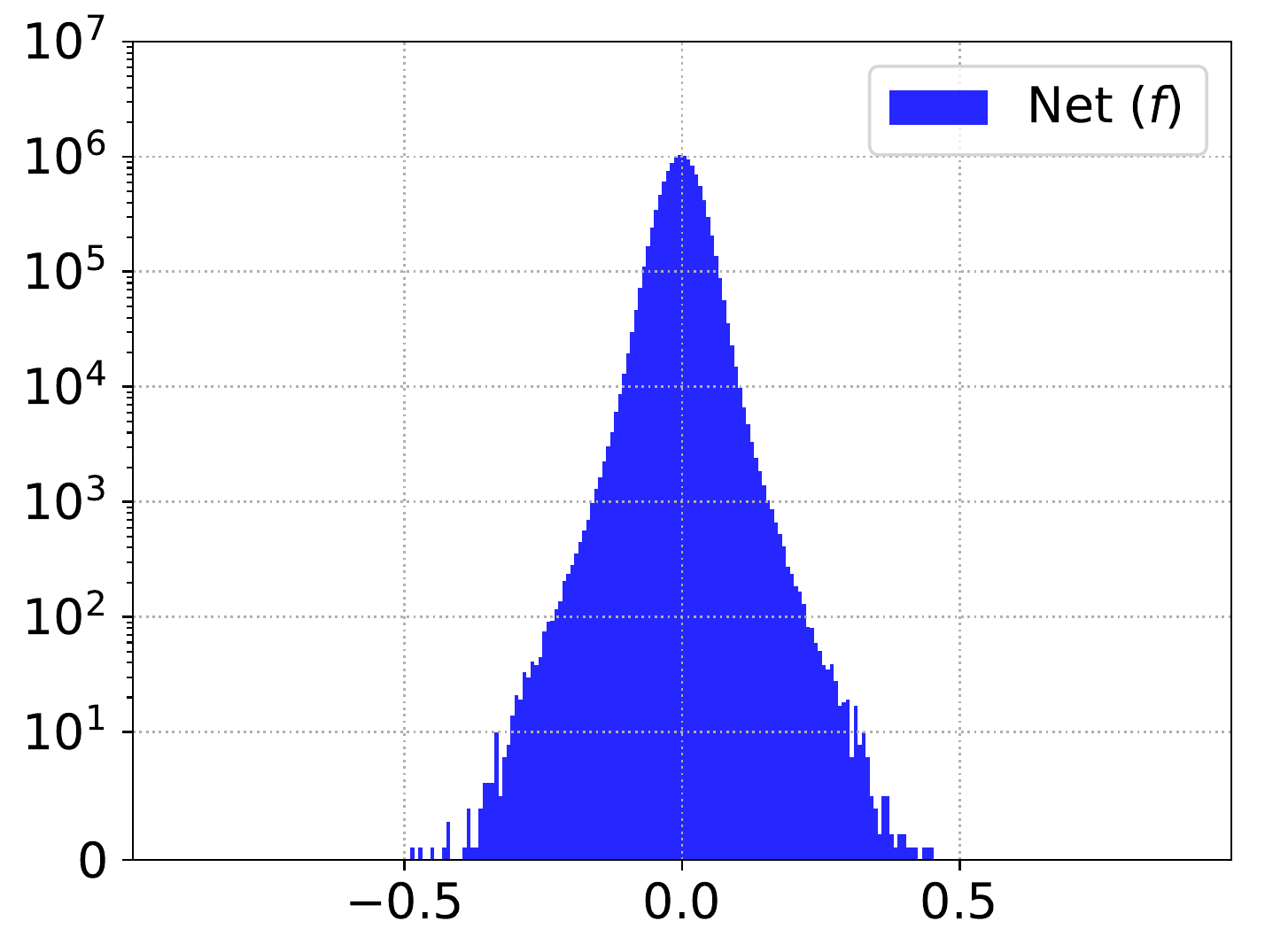}
\includegraphics[width=.495\linewidth]{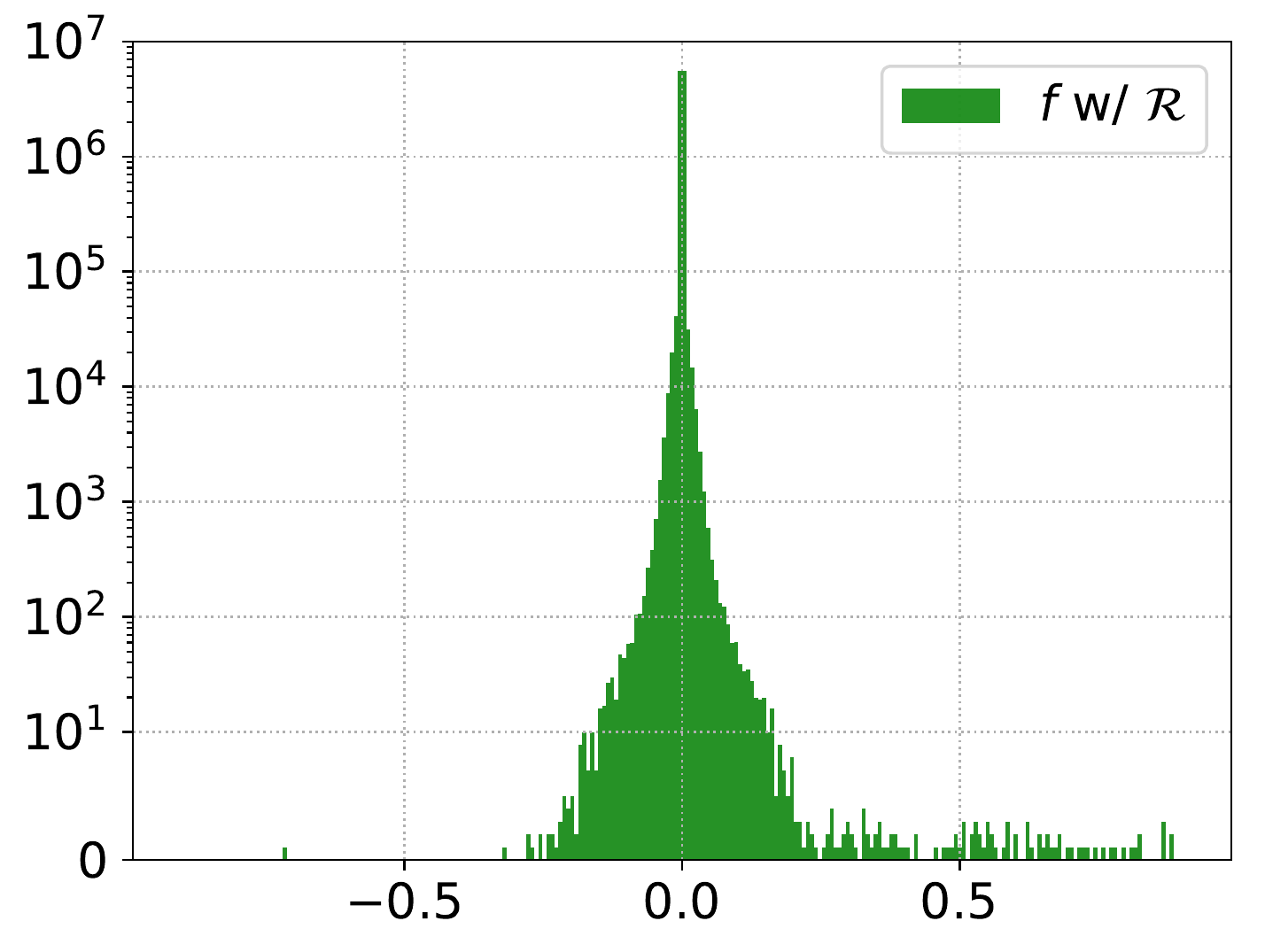}
\end{center}
\vspace{-.8em}
\caption{Histograms of the parameters learned by different regression networks on STB~\cite{zhang2017hand} dataset. Left: Histogram of the network $f$ without any form of regularization. Right: Histogram of the network trained with the proposed regularizer $\mathcal{R}$ in Eq.~\eqref{eq:reg}.}
\label{fig:param_comparison}
\end{figure}

\begin{figure*}[t]
\begin{center}
\includegraphics[width=\linewidth]{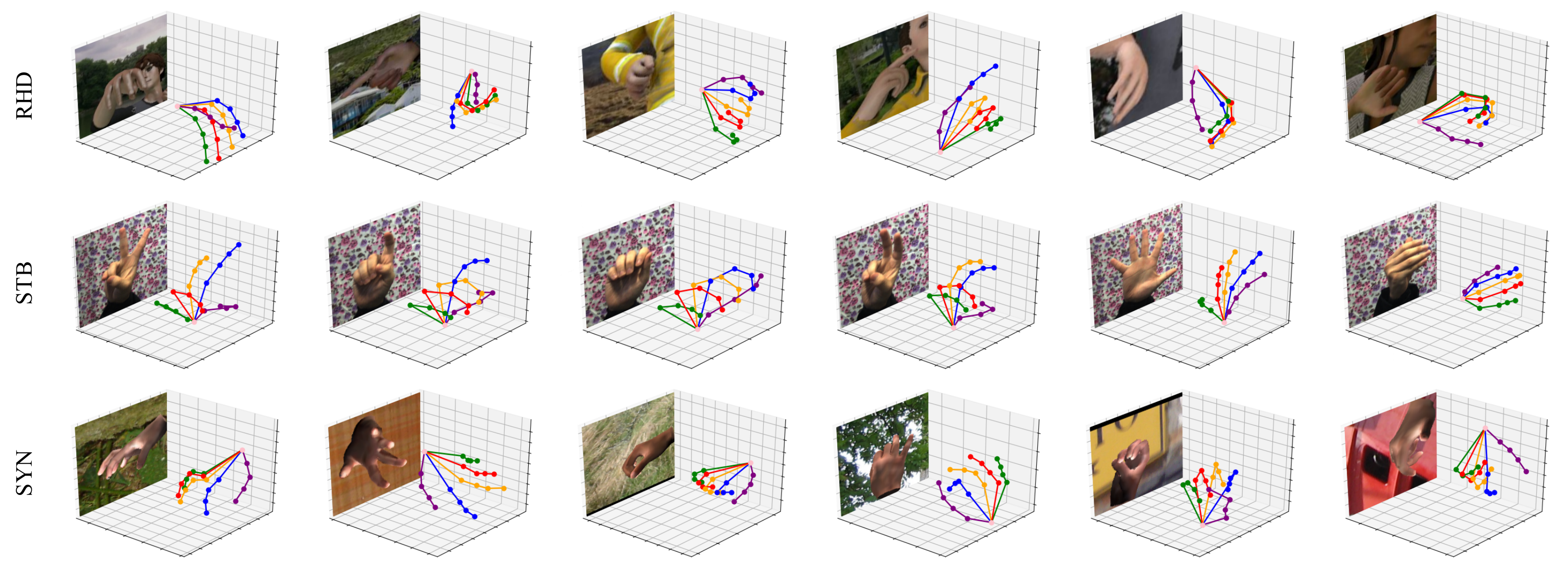}
\end{center}
\vspace{-.6em}
\caption{Visual results of our approach on RHD~\cite{zimmermann2017learning} (top) and STB~\cite{zhang2017hand} (middle). To demonstrate the generalizability of the proposed method, we also show the results after applying the network trained on STB to the synthetic dataset (bottom). Best viewed in color.}
\label{fig:visual}
\end{figure*}

\textbf{Visualization of Parameters.} To give an intuitive understanding of how our regularizer $\mathcal{R}$ affects the network learning, we plot the histograms of the parameters learned by the network with and without the use of $\mathcal{R}$ in Fig.~\ref{fig:param_comparison}. We can make the following observations. First, for the network trained with regularization, there is a sharper peak at zero. This is due to the positive $\phi_i$ in Eq.~\eqref{eq:reg} which decays the corresponding $\theta_i$ to zero. Second, on the other hand, the parameters of the network with regularization have wider spread, since they are boosted by the negative $\phi_i$.

\textbf{Comparison to State of the Art.} We further compare the proposed CMKG to other approaches~\cite{mueller2018ganerated,panteleris2018using,spurr2018cross,zimmermann2017learning} on STB in Fig.~\ref{fig:dist_comparison} (right). We can see that our regularized network matches the state-of-the-art performance without using complex network architecture, loss functions or additional constraints like previous methods. Our visual results are shown in Fig.~\ref{fig:visual}. As seen, our method is able to accurately predict 3D hand poses across different datasets and generalize the learned knowledge to some novel cases.

\subsection{Discussion}\label{sec:discussion}\label{sec:exp:dis}

One potential concern about the learned knowledge (regularizer) from the source dataset is how it performs when applied to different target datasets. First of all, we point out that it is impossible to learn a domain-independent regularizer from a single source which performs consistently well on all other datasets, since their data usually follow different statistics. Here, we hypothesize that the effect of the learned regularizer depends on two factors: (1) the domain shift between the source and target dataset, and (2) the effect of regularization on the target dataset.

The first factor is straightforward as large domain shifts always lead to difficulties in network generalization. This is a well-defined problem in transfer learning which is tackled by domain adaptation~\cite{pan2009survey,qiao2020learning,quionero2009dataset}. To illustrate the second factor, we conduct an additional experiment which applies the same regularizers in Sect.~\ref{sec:exp:gen} to a number of different target datasets. Due to the space limitation, we ask the readers to refer to the supplementary material for detailed setups of this experiment. Looking at Table~\ref{tbl:db_ablation}, we see a strong correlation between the default and the proposed regularizer: if there is a large increase obtained by the default regularizer, $\mathcal{R}$ can boost the performance even further; otherwise, our improvement is limited. This is intuitive since our formulate is consist with the default regularization technique.

\begin{table}[t]
\begin{center}
\resizebox{.85\linewidth}{!}{
\begin{tabular}{l|c|c|c}
\toprule
Target Dataset & w/o $\mathcal{R}$ & Default $\ell^2$ & $\mathcal{R}$ in Eq.~\eqref{eq:reg} \\
\midrule
FreiHAND~\cite{zimmermann2019frei}, $\mathcal{G}$ & $12.37$ & $12.28_{\downarrow 0.09}$ & $12.27_{\downarrow 0.10}$ \\
FreiHAND~\cite{zimmermann2019frei}, $\mathcal{H}$ & $14.49$ & $14.02_{\downarrow 0.47}$ & $13.82_{\downarrow 0.67}$ \\
FreiHAND~\cite{zimmermann2019frei}, $\mathcal{S}$ & $15.80$ & $14.92_{\downarrow 0.88}$ & $14.26_{\downarrow 1.54}$ \\
FreiHAND~\cite{zimmermann2019frei}, $\mathcal{A}$ & $16.18$ & $15.16_{\downarrow 1.02}$ & $14.18_{\downarrow 2.00}$ \\
STB~\cite{zhang2017hand} & $15.67$ & $12.02_{\downarrow 3.65}$ & $\;\: 8.18_{\downarrow 7.49}$ \\
\bottomrule
\end{tabular}}
\end{center}
\caption{Effect of regularizers on different target datasets. We report EPE (mm) and the performance gain for each setting. $\mathcal{G}$, $\mathcal{H}$, $\mathcal{S}$ and $\mathcal{A}$ are four different domains contained in FreiHAND~\cite{zimmermann2019frei}.}
\label{tbl:db_ablation}
\end{table}

Our findings suggest multiple directions of future work. For one, the proposed scheme currently has access to only one single source dataset; we believe that learning from multiple sources will result in better generalizability of the model. On the other hand, we treat target priors as a regularization term in this work, which is perhaps the simplest formulation. We believe that a further exploration on choices of this term will result in improved performance.

\section{Conclusion}\label{sec:conclusion}

We introduce an end-to-end scheme for Cross-Modal Knowledge Generalization to transfer cross-modal knowledge between source and target datasets where superior modalities are missing. The core idea is to interpret knowledge as priors on the parameters of the student network which can be efficiently learned by meta-learning. Our method is comprehensively evaluated in 3D hand pose estimation. We show that our scheme can efficiently generalize cross-modal knowledge to the target dataset and significantly boost the network to match the state-of-the-art performance. We believe our work provides new insights in conventional cross-modal knowledge distillation tasks, and serves as a strong baseline in this novel research direction.

{\small
\bibliographystyle{ieee_fullname}
\bibliography{egbib}
}

\clearpage
\begin{appendices}
\section{Supplementary Material}

\subsection{Proof of Proposition 1}\label{supp:proof}

\begin{proposition}[Proposition~\ref{thm:elbo} restated]
\label{thm:elbo_supp}
Let $q$ be any posterior distribution function over the latent variables $\boldsymbol{\theta}$ given the evidence $\mathcal{D}_S$. Then, the marginal log-likelihood can be lower bounded:
\begin{equation*}
\log P(\mathcal{D}_S|\boldsymbol{\phi}) = \log \int P(\mathcal{D}_S, \boldsymbol{\theta}|\boldsymbol{\phi}) d\boldsymbol{\theta} \geq \mathcal{E}(q,\boldsymbol{\phi}),
\end{equation*}
where $\mathcal{E}$ is the evidence lower-bound (ELBO) defined as:
\begin{equation*}
\mathcal{E}(q,\boldsymbol{\phi}) \triangleq \mathbb{E}_q [ \log P(\mathcal{D}_S|\boldsymbol{\theta})]\\ - \textup{KL}[ q(\boldsymbol{\theta}|\mathcal{D}_S) \| P(\boldsymbol{\theta}|\boldsymbol{\phi})].
\end{equation*}
\end{proposition}

\begin{proof}

The proposed meta-training as described in Algorithm 1 of the main paper makes a posterior inference based on the graphical model in Fig.~\ref{fig:meta_graph}. Given the evidence $\mathcal{D}_S$, learning the parameters $\boldsymbol{\phi}$ leads to maximize the likelihood $P(\mathcal{D}_S|\boldsymbol{\phi})$:
\begin{equation*}
\begin{aligned}
\log P(\mathcal{D}_S|\boldsymbol{\phi}) &= \log \int P(\mathcal{D}_S, \boldsymbol{\theta}|\boldsymbol{\phi}) d\boldsymbol{\theta}\\
&= \log \int P(\mathcal{D}_S|\boldsymbol{\theta}, \boldsymbol{\phi}) P(\boldsymbol{\theta}|\boldsymbol{\phi}) d\boldsymbol{\theta}\\
&= \log \int P(\mathcal{D}_S|\boldsymbol{\theta}) P(\boldsymbol{\theta}|\boldsymbol{\phi}) d\boldsymbol{\theta}\\
&= \log \int q(\boldsymbol{\theta}|\mathcal{D}_S) \frac{P(\mathcal{D}_S|\boldsymbol{\theta}) P(\boldsymbol{\theta}|\boldsymbol{\phi})}{q(\boldsymbol{\theta}|\mathcal{D}_S)} d\boldsymbol{\theta}.
\end{aligned}
\end{equation*}
By Jensen's inequality, we have:
\begin{equation*}
\begin{aligned}
\log P(\mathcal{D}_S|\boldsymbol{\phi}) &= \log \int q(\boldsymbol{\theta}|\mathcal{D}_S) \frac{P(\mathcal{D}_S|\boldsymbol{\theta}) P(\boldsymbol{\theta}|\boldsymbol{\phi})}{q(\boldsymbol{\theta}|\mathcal{D}_S)} d\boldsymbol{\theta}\\
&\geq \int q(\boldsymbol{\theta}|\mathcal{D}_S) \log \frac{P(\mathcal{D}_S|\boldsymbol{\theta}) P(\boldsymbol{\theta}|\boldsymbol{\phi})}{q(\boldsymbol{\theta}|\mathcal{D}_S)} d\boldsymbol{\theta}\\
&\triangleq \mathcal{E}(q, \boldsymbol{\phi}),
\end{aligned}
\end{equation*}
where $\mathcal{E}(q, \boldsymbol{\phi})$ is the evidence lower-bound (ELBO) of the likelihood $\log P(\mathcal{D}_S|\boldsymbol{\phi})$. Then, we further have:
\begin{equation*}
\begin{aligned}
\mathcal{E}(q, \boldsymbol{\phi}) &= \int q(\boldsymbol{\theta}|\mathcal{D}_S) \log \frac{P(\mathcal{D}_S|\boldsymbol{\theta}) P(\boldsymbol{\theta}|\boldsymbol{\phi})}{q(\boldsymbol{\theta}|\mathcal{D}_S)} d\boldsymbol{\theta}\\
&= \begin{aligned}[t] \int q(\boldsymbol{\theta}|\mathcal{D}_S) &\log P(\mathcal{D}_S|\boldsymbol{\theta}) d\boldsymbol{\theta}\\
&+ \int q(\boldsymbol{\theta}|\mathcal{D}_S) \log \frac{P(\boldsymbol{\theta}|\boldsymbol{\phi})}{q(\boldsymbol{\theta}|\mathcal{D}_S)} d\boldsymbol{\theta}\end{aligned}\\
&= \begin{aligned}[t] \int q(\boldsymbol{\theta}|\mathcal{D}_S) &\log P(\mathcal{D}_S|\boldsymbol{\theta}) d\boldsymbol{\theta}\\
&- \int q(\boldsymbol{\theta}|\mathcal{D}_S) \log \frac{q(\boldsymbol{\theta}|\mathcal{D}_S)}{P(\boldsymbol{\theta}|\boldsymbol{\phi})} d\boldsymbol{\theta}\end{aligned}\\
&= \begin{aligned}[t] \mathbb{E}_{\boldsymbol{\theta} \sim q(\boldsymbol{\theta}|\mathcal{D}_S)} &\left[ \log P(\mathcal{D}_S|\boldsymbol{\theta}) \right]\\
&- \text{KL}\left[ q(\boldsymbol{\theta}|\mathcal{D}_S) \middle\| P(\boldsymbol{\theta}|\boldsymbol{\phi}) \right].\end{aligned}
\end{aligned}
\end{equation*}
We have thus proven Proposition~\ref{thm:elbo_supp}.
\end{proof}

\begin{figure}[t]
\begin{center}
\includegraphics[width=0.36\linewidth]{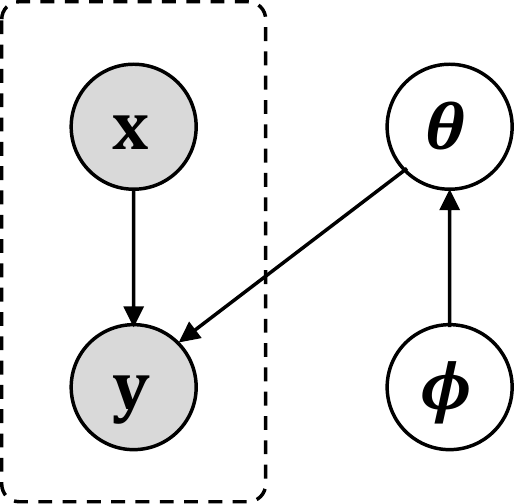}
\end{center}
\caption{Graphical models for meta-training algorithm.}
\label{fig:meta_graph}
\end{figure}

\subsection{Efficient Implementation of Algorithm 1}\label{supp:implement}

In order to implement Algorithm~\ref{alg:meta_training} of the main paper, we need to compute the second order derivative of the network parameters when a set of $\boldsymbol{\phi}$ are updated by gradient descent. This is computational expensive especially when the scale of the backbone network becomes very large. In this section, we provide an efficient implementation of Algorithm~\ref{alg:meta_training} when the derivative \wrt $\boldsymbol{\phi}$ of the regularizer $\mathcal{R}$ can be calculated directly.

As described in the main paper, we implement $\mathcal{R}$ by a weighted $\ell^2$ regularizer in this work. Therefore, the regularized objective function of Eq.~\eqref{eq:meta_reg} in the main paper can be rewritten by:
\begin{equation}
\label{eq:loss}
\mathcal{F}(\mathbf{x}_i, \mathbf{y}_i; \boldsymbol{\theta}, \boldsymbol{\phi}) = \mathcal{L}_\text{REG}(\mathbf{x}_i, \mathbf{y}_i; \boldsymbol{\theta}) + \sum_i \phi_i \|\theta_i\|^2,
\end{equation}
where $\phi_i$ is the $i$-th weight of the regularizer and $\theta_i$ is the $i$-th parameter of the student network. Then, the $k$-th gradient descent step of the network parameter $\theta_i^k$ is:
\begin{equation}
\begin{aligned}
\label{eq:grad}
\theta_i^{k+1} &= \theta_i^k - \alpha\frac{\partial \mathcal{F}}{\partial \theta_i^k} = \theta_i^k - \alpha\frac{\partial \left( \mathcal{L}_\text{REG} + \sum_i \phi_i \|\theta_i^k\|^2 \right)}{\partial \theta_i^k}\\
&= \theta_i^k - \alpha\frac{\partial \mathcal{L}_\text{REG}}{\partial \theta_i} - 2\alpha\phi_i\theta_i^k\\
&= \theta_i^k (1 - 2\alpha\phi_i) - \alpha\frac{\partial \mathcal{L}_\text{REG}}{\partial \theta_i^k},
\end{aligned}
\end{equation}
where $\alpha$ is the learning rate of $\theta_i$. We can see that Eq.~\eqref{eq:grad} converts our regularizer formulation into the weight decay mechanism, where $2\phi_i$ turns into the decay rate. Since the second term of Eq.~\eqref{eq:grad} is independent with $\phi_i$, we only need to compute the first order derivative when updating $\phi_i$ of the regularizer $\mathcal{R}$. The modified meta-training approach is illustrated in Algorithm~\ref{alg:meta_training_e}.

\begin{algorithm}[t]
\begin{algorithmic}
\Require Batch size $N$, \# of iterations $K$, learning rate $\alpha$.
\Require \# of inner iterations $l$, meta learning rate $\beta$.
\State Initialize $\boldsymbol{\theta}_0$, $\boldsymbol{\phi}_0$
\For{$k=0$ to $K-1$}
    \State Sample $N$ examples $\{(\mathbf{x}^S_n, \tilde{\mathbf{x}}^S_n, \mathbf{y}^S_n) \sim \mathcal{D}_S\}_{n = 1}^N$
    \State $\ddot{\boldsymbol{\theta}}_0 \gets \boldsymbol{\theta}_k$
    \For{$i=0$ to $l-1$}
        \State $\ddot{\boldsymbol{\theta}}_{i+1} \gets \ddot{\boldsymbol{\theta}}_i (1 - 2\alpha\boldsymbol{\phi}_k) - \alpha\nabla_{\ddot{\boldsymbol{\theta}}_i}\mathcal{L}_\text{REG}(\mathbf{x}^S_n, \mathbf{y}^S_n; \ddot{\boldsymbol{\theta}}_i)$
    \EndFor
    \State $\ddot{\boldsymbol{\theta}}_k \gets \ddot{\boldsymbol{\theta}}_l$
    \State $\boldsymbol{\phi}_{k+1} \gets \boldsymbol{\phi}_k - \beta\nabla_{\boldsymbol{\phi}_k}\mathcal{G}(\mathbf{x}^S_n, \tilde{\mathbf{x}}^S_n, \mathbf{y}^S_n; \ddot{\boldsymbol{\theta}}_k)$
    \State $\boldsymbol{\theta}_{k+1} \gets \boldsymbol{\theta}_k - \alpha\nabla_{\boldsymbol{\theta}_k}\mathcal{G}(\mathbf{x}^S_n, \tilde{\mathbf{x}}^S_n, \mathbf{y}^S_n; \boldsymbol{\theta}_k)$
\EndFor
\State $\boldsymbol{\phi}_{\text{META}} \gets \boldsymbol{\phi}_K$
\end{algorithmic}
\caption{Efficient implementation of meta-training.}
\label{alg:meta_training_e}
\end{algorithm}

\subsection{Experimental Setup on FreiHAND}\label{supp:freihand}

FreiHAND~\cite{zimmermann2019frei} is a 3D hand pose dataset which records different hand actions performed by 32 people. For each hand image, MANO-based 3D hand pose annotations are provided. It currently contains 32,560 unique training samples and 3960 unique samples for evaluation. The training samples are recorded with a green screen background allowing for background removal. In addition, it applies three different post processing strategies to training samples for data augmentation. However, these post processing strategies are not applied to evaluation samples.

In Sect.~\ref{sec:exp:dis} of the main paper, we conduct the experiment to evaluate the performance of the learned regularizer when it is applied to different target datasets (domains). In this experiment, we treat the original images collected with the green screen background ($\mathcal{G}$) in FreiHAND, together with their post-processed results using three different strategies: harmonization~\cite{tsai2017deep} ($\mathcal{H}$), colorization auto~\cite{zhang2017real} ($\mathcal{A}$), colorization sample~\cite{zhang2017real} ($\mathcal{S}$), as three different domains contained by FreiHAND. However, since the domains of $\mathcal{H}$, $\mathcal{A}$ and $\mathcal{S}$ are not provided for the original evaluation samples, we create new training and evaluation splits from the original training data of FreiHAND. Therefore, for each domain, the first 30,000 training samples are used for network training while the rest 2,560 samples are leveraged for evaluation. We use the same setting as described in Sect.~\ref{sec:exp:imp} of the main paper to train the network in this dataset.

\subsection{Additional Visual Results}\label{supp:visual}

In Figs.~\ref{fig:vis_rhd} to \ref{fig:vis_syn}, we show additional visual results predicted by our method on RHD~\cite{zimmermann2017learning}, STB~\cite{zhang2017hand} and the synthetic dataset. We can see that our method is able to accurately estimate 3D hand poses across different datasets.

\begin{figure*}[t]
\begin{center}
\includegraphics[width=1.\linewidth]{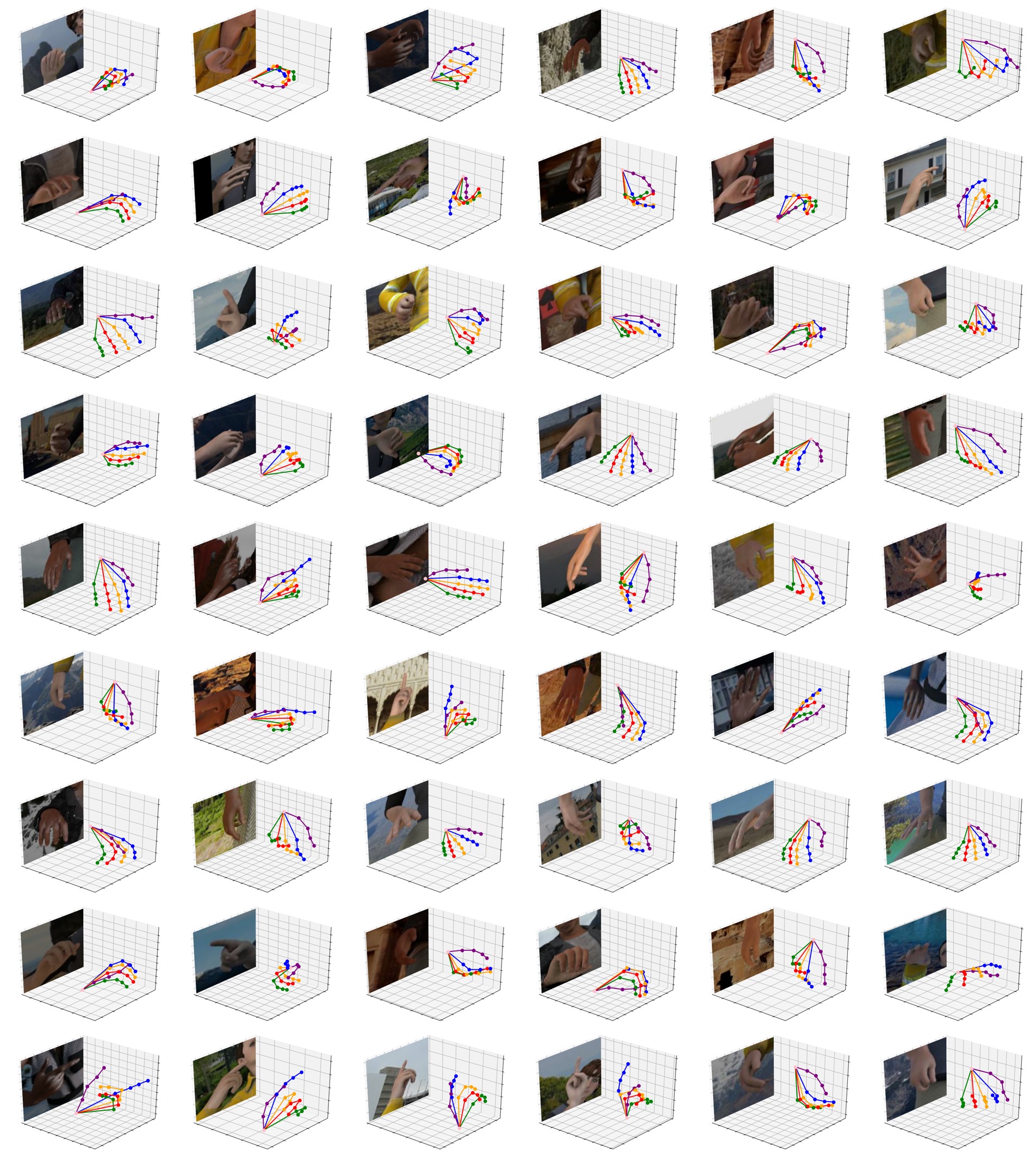}
\end{center}
\caption{Additional visual results of our approach on RHD~\cite{zimmermann2017learning} dataset.}
\label{fig:vis_rhd}
\end{figure*}

\begin{figure*}[t]
\begin{center}
\includegraphics[width=1.\linewidth]{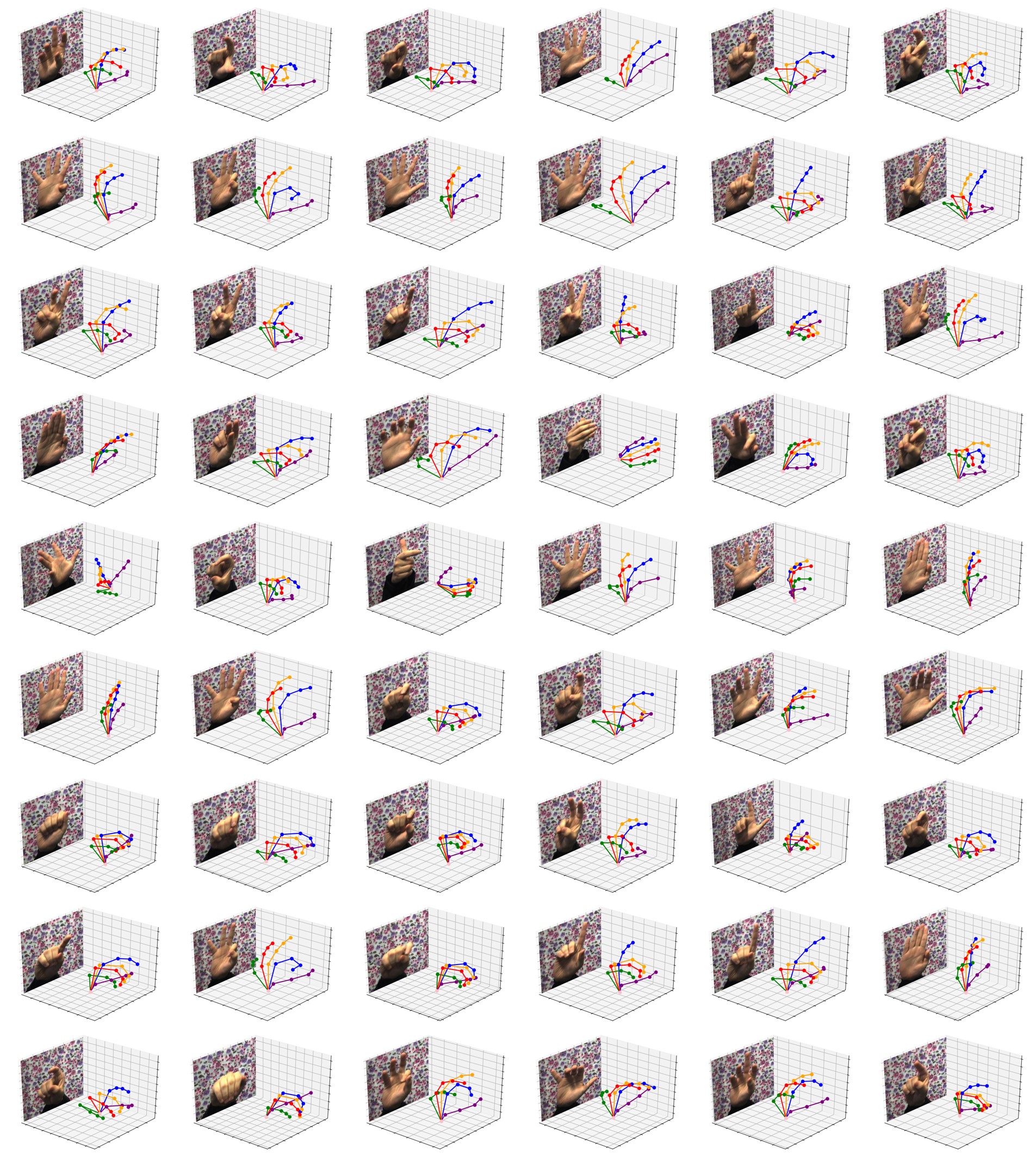}
\end{center}
\caption{Additional visual results of our approach on STB~\cite{zhang2017hand} dataset.}
\label{fig:vis_stb}
\end{figure*}

\begin{figure*}[t]
\begin{center}
\includegraphics[width=1.\linewidth]{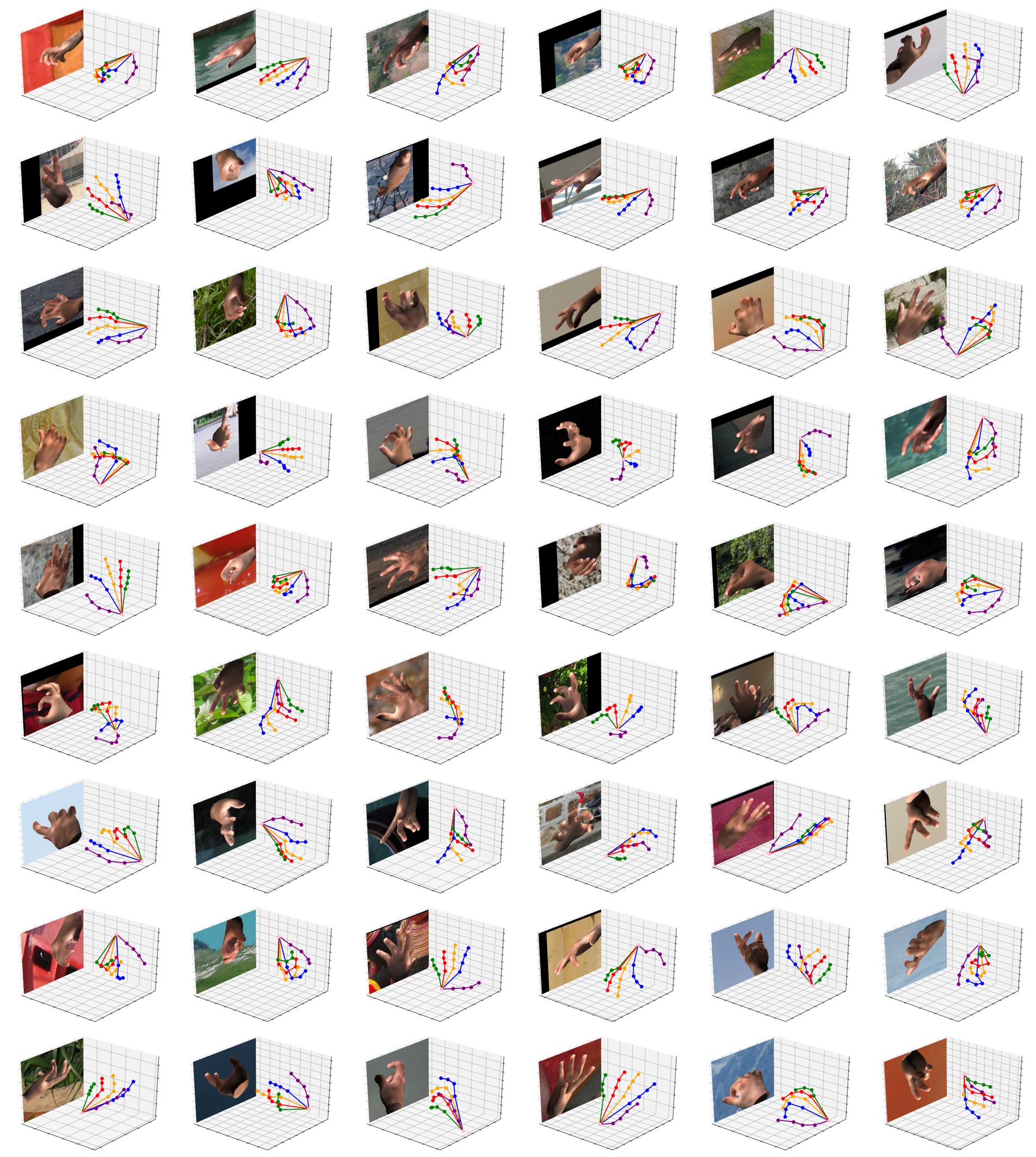}
\end{center}
\caption{Additional visual results of our approach on synthetic dataset.}
\label{fig:vis_syn}
\end{figure*}

\end{appendices}

\end{document}